\documentclass{article}
\usepackage{ijcai22}

\usepackage{times}
\usepackage{soul}
\usepackage{url}
\usepackage[hidelinks]{hyperref}
\usepackage[utf8]{inputenc}
\usepackage[small]{caption}
\usepackage{graphicx}
\usepackage{amsmath}
\usepackage{amsthm}
\usepackage{booktabs}
\usepackage{algorithm}
\usepackage{algorithmic}

\usepackage{amssymb}
\usepackage{multirow}
\usepackage{multicol}
\usepackage{subfig}
\newtheorem{theorem}{Theorem}
\newtheorem{definition}{Definition}
\newtheorem{proposition}[theorem]{Proposition}

\title{Domain Adaptation via Maximizing Surrogate Mutual Information}

\author{
Haiteng Zhao
\and
Chang Ma\and
Qinyu Chen\And
Zhi-Hong Deng
\affiliations
Peking University\\
\emails
\{zhaohaiteng,changma,chenqinyu,zhdeng\}@pku.edu.cn
}

\begin{document}

\maketitle

\begin{abstract}
Unsupervised domain adaptation (UDA) aims to predict unlabeled data from target domain with access to labeled data from the source domain. In this work, we propose a novel framework called SIDA (Surrogate Mutual Information Maximization Domain Adaptation) with strong theoretical guarantees. To be specific, SIDA implements adaptation by maximizing mutual information (MI) between features. In the framework, a surrogate joint distribution models the underlying joint distribution of the unlabeled target domain. Our theoretical analysis validates SIDA by bounding the expected risk on target domain with MI and surrogate distribution bias. Experiments show that our approach is comparable with state-of-the-art unsupervised adaptation methods on standard UDA tasks.
\end{abstract}

\section{Introduction}
Inspired by human beings’ ability to transfer knowledge across domains and tasks, transfer learning is proposed to leverage knowledge from source domain and task to improve  performance on target domain and task. 
However, in practice, labeled data are often limited on target domains. To address such situation, unsupervised domain adaptation (UDA), a category of transfer learning methods \cite{long2015learning,long2017deep,ganin2016domain}, 
 attempts to enhance knowledge transfer from labeled source domain to target domain by leveraging unlabeled target domain data.

Most previous work is based on the data shift assumption, \emph{i.e.}, the label space maintains the same across domains, but the data distribution conditioned on labels varies. Under this hypothesis, domain alignment and class-level method are used to improve generalization across source and target feature distributions.
Domain alignment minimizes the discrepancy between the feature distributions of two domains \cite{long2015learning,ganin2016domain,long2017deep}, while class-level methods work on conditional distributions. Conditional alignment aligns conditional distributions and use pseudo-labels to estimate conditional distribution on target domain \cite{long2018conditional,li2020domain,chen2020adversarial}. However, the conditional distributions from different categories tend to mix together, leading to performance drop. Contrastive learning based methods resolve this issue by discriminating features from different classes \cite{luo2020unsupervised}, but still face the problem of pseudo-label precision.
In addition, most of the class-level methods lack solid theoretical explanations for the relationship between cross domain generalization and their objectives.
Some works \cite{chen2019progressive,xie2018learning} yield some intuition for conditional alignment and contrastive learning, but the relation between their training objectives and cross-domain error remains unclear.

In this work, we aim to address the generalization problem in domain adaptation from an information theory perspective. In failed case of domain adaptation, as shown in Figure \ref{motivation},  features from the same class do not represent each other well and this inspires us to use mutual information to reduce this confusion. Our motivation is to find more representative features for both domains by maximizing mutual information between features of the same class (on both source and target domains). Therefore, if our classifier can accurately predict features on source domain, then it would also function well on target domains where features share enough information with the source features.



 Based on the above motivation, we propose Surrogate Information Domain Adaptation (SIDA), a general domain adaptation framework with strong theoretical guarantees. SIDA achieves adaptation by maximizing the mutual information (MI) between features within the same class, which improves the generalization of the model to the target domain. Furthermore, a surrogate distribution is constructed to approximate the unlabeled target distribution, which improves flexibility for selecting data and assists MI estimation. Also, our theoretical analyses directly establish a bound between MI of features and target expected risk, giving a proof that our model can improve generalization across domain.

Our novelties and contributions are summarized as follows:
\begin{itemize}
    
    \item We propose a novel framework to achieve domain adaptation by maximizing surrogate MI.
    
    \item We establish an expected risk upper bound based on feature MI and surrogate distribution bias for UDA. This provides theoretical guarantee for our framework.
    
    \item Experiment results on three challenging benchmarks demonstrate that our method performs favorably against state-of-art class-level UDA models.
    
\end{itemize}

\begin{figure}[htbp]
\centering
    \vspace{-1pt} 
	\includegraphics[width=0.8\linewidth]{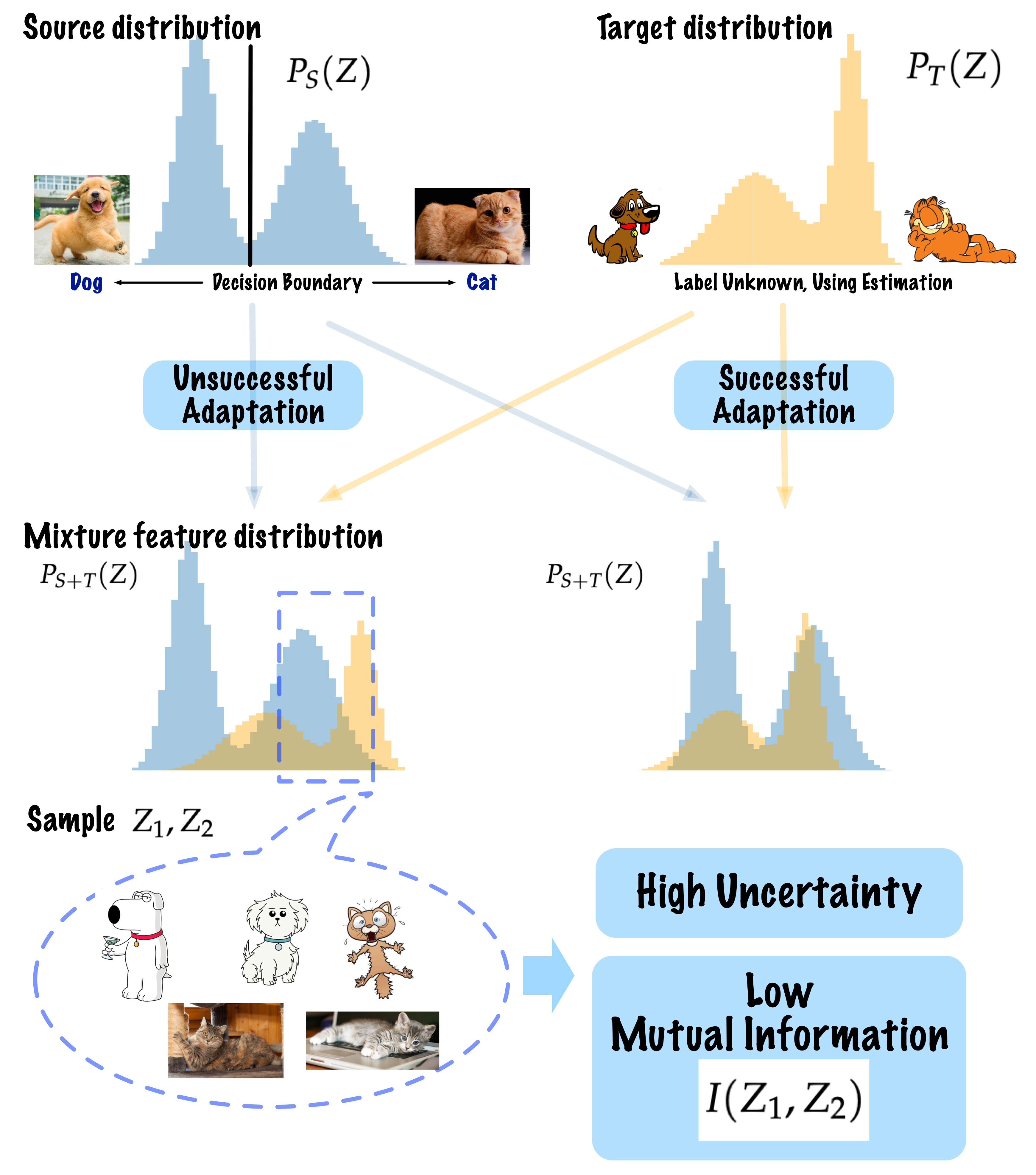}
	\caption{Our motivation is that higher intra-class feature mutual information implies better generalization ability across domain. As shown, unsuccessful domain adaptation implies lower intra-class feature mutual information. This is reflected in the fact that features from the same class do not represent each other well because of confusion with features from other classes.}
	\label{motivation}
\end{figure}

\section{Related Work}

\textbf{Domain Adaptation}\ \ 
Prior works are based on two major assumptions: (1) the label shift hypothesis, where the label distribution changes, and (2) a more common data shift hypothesis where we only study the shift in conditional distribution under the premise that the label distribution is fixed. 
Our work focuses on the data shift hypothesis, and previous work following this line can be divided into two major categories: domain alignment methods which align marginal distributions, and class-level methods addressing the alignment of conditional distributions.

Domain alignment methods minimize the difference between feature distributions of source and target domains with various metrics, e.g. maximum mean discrepancy (MMD) \cite{long2015learning}, JS divergence \cite{ganin2016domain} estimated by adversarial discriminator, Wasserstein metric and others. Maximum Mean Discrepancy (MMD) is applied to measure the discrepancy in marginal distributions \cite{long2015learning,long2017deep}. Adversarial domain adaptation plays a mini-max game to learn domain-invariant features \cite{ganin2016domain,li2020model}.

Class-level methods align the conditional distribution based on pseudo-labels \cite{li2020domain,chen2020adversarial,luo2020unsupervised,li2020enhanced,tang2020discriminative,xu2019larger}. Conditional alignment methods  \cite{xie2018learning,long2018conditional} minimize the discrepancy between conditional distributions.  In class-level methods, conditional distributions are assigned by pseudo-labels. The accuracy of pseudo-labels greatly influences performance and later works construct more accurate pseudo-labels \cite{chen2020adversarial}. However, the major problem with this method is that error in conditional alignment leads to distribution overlap of features from different class, resulting in low discriminability on target domain.  Contrastive learning addresses this problem by maximizing the discrepancy between different classes \cite{luo2020unsupervised,li2020enhanced}. However, the performance of contrastive learning also relies on pseudo-labeling.

In addition, previous class-level works provide weak theoretical support for cross-domain generalization.
Prior works mainly focus on domain alignment \cite{ben2007analysis,redko2020survey}. Some works
\cite{chen2019progressive,xie2018learning} consider optimal classification on both domains, and yield some intuitive explanation for conditional alignment and contrastive learning, but the relation between their objective function and theoretical cross-domain error remains unclear.

\textbf{Information Maximization Principle}\ \ Recently, mutual information maximization (InfoMax) for representation learning has attracted lots of attention \cite{chen2020simple,hjelm2018learning,khosla2020supervised}. The intuition is that two features belonging to different classes should be discriminable while features of the same class should resemble each other. The InfoMax principle provides a general framework for learning informative representations, and provides consistent boosts in various downstream tasks.

We facilitate domain adaptation with MI maximization, i.e. maximizing the MI between features of the same class. Some works solve domain adaptation problem via information theoretical methods \cite{thota2021contrastive,chen2020structure,park2020joint}, which maximize MI using InfoNCE estimation \cite{poole2019variational}. As far as we know, we are the first to provide theoretical guarantee for the target domain expected risk based on MI. Compared with InfoNCE, the variational lower bound of MI we use is tighter \cite{poole2019variational}. We also construct a surrogate distribution as a substitute for unlabeled target domain, which is more suitable for MI estimation.

\begin{figure*}[htbp]
    \vspace{-1pt} 
	\includegraphics[width=\textwidth]{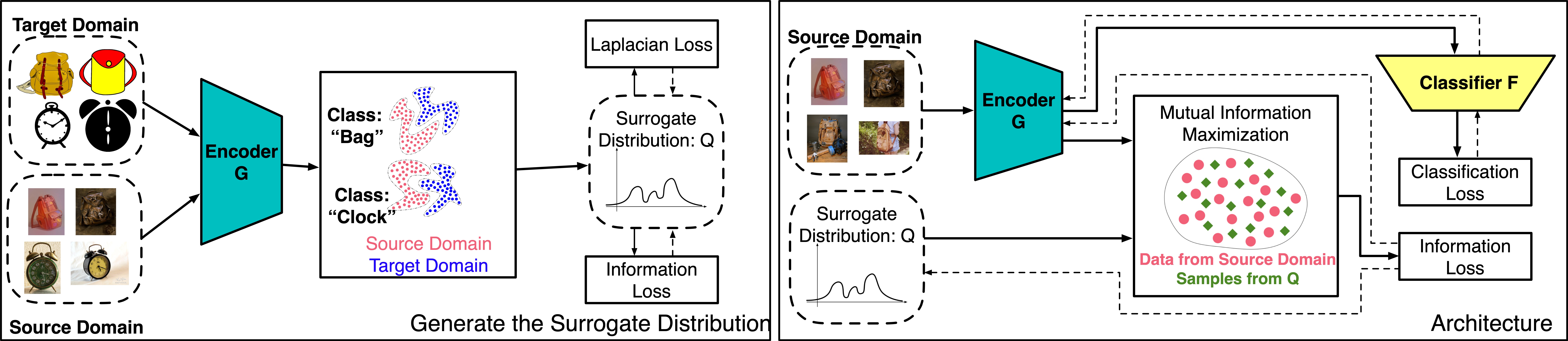}
	\caption{Overview of SIDA framework for training. Only encoder and classifier are involved in inference. The dashed arrow shows the path  of the gradient backpropagation.}
	\label{overview}
\end{figure*}

\section{Preliminaries}
\subsection{Notations and Problem Setting}
Let $\mathcal{X}$ be the data space and $\mathcal{Y}$ be the label space. In UDA, there
is a source distribution $P_S(X,Y)$ and a target distribution $P_T(X,Y)$ on $\mathcal{X} \times \mathcal{Y}$. Note that distributions are also referred to as domains in UDA. Our work is based on the data shift hypothesis, which assumes $P_S(X,Y)$ and $P_T(X,Y)$ satisfy the following properties: $P_T(Y)= P_S(Y)$ and $P_T(X|Y)\not = P_S(X|Y)$ . 

In our work, we focus on classification tasks. Under this setting, an algorithm has access to $n_S$ labeled samples $\{(x_S^i,y_S^i)\}_{i=1}^{n_S} \sim P_S(X,Y)$ and $n_T$ unlabeled samples $\{(x_T^i)\}_{i=1}^{n_T} \sim P_T(X)$, and outputs a hypothesis composed of an encoder $G$ and a classifier $F$. Let $\mathcal{Z}$ be the feature space. The encoder maps data to feature space, denoted by $G : \mathcal{X} \rightarrow \mathcal{Z}$. Then the classifier maps the feature to a corresponding class, $F: \mathcal{Z} \rightarrow \mathcal{Y}$. 

For brevity, given encoder $G$ and data-label distribution $P(X,Y)$, denote the distribution of $G$-encoded feature and label by $P^G$, i.e. $P^G(z,y)= P(x=G^{-1}(z),y)$.

Let $F$ be a hypothesis, and $P$ be the distribution of feature and label. The expected risk of a F w.r.t. P is denoted as
\begin{equation}
    \setlength{\abovedisplayskip}{4pt}
        \setlength{\belowdisplayskip}{0cm}
    \epsilon_P (F)\triangleq\mathbb{E}_{P(z)}|\delta_{F(z)}-P(y|z)|_1,
\end{equation}

  where $\delta_{F(z)}(y)$ equals to 1 if $y=F(z)$ and equals 0 in else cases.  Our objective is to minimize the expected risk of $F$ on target feature distribution encoded by $G$,
\begin{equation}
\setlength{\abovedisplayskip}{4pt}
        \setlength{\belowdisplayskip}{0cm}
    \min_{G,F} \epsilon_{P^G_T}(F).
\end{equation}

\section{Methodology}
\subsection{Overview}
In UDA task, the model needs to generalize across different domains with varying distributions; thus the encoder needs to extract appropriate features that are transferable across domains. The challenges of class-level adaptation are two folds: learning transferable features, and modeling $P^G_T(Z|Y)$ without label information.

To solve the first problem, we use MI based methods. Following the InfoMax principle, we maximize the mutual information between features from the same class on the target and source mixture distribution. This encourages the features of the source domain to carry more information about the features of the same class in target domain, and thus provides opportunities for transferring classifier across domains.

As for the second challenge, we first revisit the data shift hypothesis. The distribution of labels $P(Y)$ remains independent of domains; therefore the key is to model the conditional distribution $P(Z|Y)$ on the target domain. However, modeling $P(Z|Y)$ is intractable, since labels on the target domain are inaccessible. To tackle this problem, we model a surrogate distribution $Q(Z|Y)$ instead.



We introduce the goal of maximizing MI in section \ref{MI goal}, and theoretically explain how MI affects domain adaptation risk. In Section \ref{SIDA Framework}, we will introduce the model in detail, including the variational estimation  of MI, the modeling of the surrogate distribution, and the optimization of the loss function of the model.

\subsection{Mutual Information Maximization}\label{MI goal}
MI measures the degree to which two variables can predict each other. Inspired by InfoMax principle \cite{hjelm2018learning}, we maximize the MI between the features within the same class. It encourages features from different classes to be discriminable from each other. 

We maximize MI between features on both source and target domain, regardless of which domain they come from. So we introduce mixture distribution  $S+T$ of both domain, which is 
\begin{small}
\begin{equation}
\setlength{\abovedisplayskip}{4pt}
        \setlength{\belowdisplayskip}{0cm}
P_{S+T}(x,y) \triangleq \frac{1}{2}(P_S(x,y) + P_T(x,y)).
\end{equation}
\end{small}

Note that because $P_S(y)=P_T(y)=P_{S+T}(y)$, $P_{S+T}(x|y)=\frac{1}{2}P_{S}(x \mid y)+\frac{1}{2}P_{T}(x \mid y)$. Define the distribution of features from the same class as 

\begin{small}
\begin{equation}
\begin{aligned}
\setlength{\abovedisplayskip}{4pt}
        \setlength{\belowdisplayskip}{0cm}
P_{S+T}^G(z_1,z_2|y)\triangleq P_{S+T}^G(z_1|y)P_{S+T}^G(z_2|y), \quad\\ P_{S+T}^G(z_1,z_2)=\sum_y P_{S+T}^G(y) P_{S+T}^G(z_1,z_2|y).
\end{aligned}
\end{equation}
\end{small}

which means the feature $z_1$ and $z_2$ are sampled independently from the conditional distribution of the same class, with equal probability from source domain and target domain.

MI between features is maximized within the mixture distribution, as formalized bellow:
\begin{small}
\begin{equation}
\begin{aligned}
    \arg\max_{G} & I^G_{S+T}(Z_1; Z_2)\\=&\int P_{S+T}^G(z_1,z_2) \log \frac{P_{S+T}^G(z_1,z_2)}{P_{S+T}^G(z_1)P_{S+T}^G(z_2)}dz_1dz_2.
    \end{aligned}
\end{equation}
\end{small}

However, due to the lack of target domain labels, $P^G_{S+T}$ is hard to model and thereby it is infeasible to estimate $I^G_{S+T}$ directly. To address this problem, we propose a surrogate joint distribution $Q(Z,Y)$ as the substitute for target domain $P^G_T$. Then the mixture distribution becomes $P_{S+Q}^G=\frac{1}{2}(P^G_S+Q)$, and the objective becomes maximizing $I_{S+Q}^G(Z_1; Z_2)$. The construction and optimization of the surrogate joint distribution is explained in Section \ref{sd}.

\subsubsection{Theoretical Motivation for MI Maximization}
We use theoretical bound to demonstrate the motivation for using MI maximization. Our theoretical results prove that minimizing the expected risk on the target domain can be naturally transformed into MI maximization and expected risk minimization on the source domain, which explains why MI maximization is pivotal to our framework. The proofs are in appendix. 

\begin{definition}[$\mathcal{H} \Delta \mathcal{H}$-Divergence]
Let  $F_1\in \mathcal{H},F_2\in \mathcal{H}$ be two hypotheses in hypothesis space $\mathcal{H}:\mathcal{Z} \rightarrow \mathcal{Y}$. Define $ \epsilon_P(F_1, F_2)$ as the disagreement between hypotheses $F_1, F_2$ w.r.t. distribution $P$ on $\mathcal{Z}$, $\epsilon_{P}\left(F_{1}, F_{2}\right)\triangleq\mathbb{E}_{z \sim P}\left|\delta_{F_{1}(z)}-\delta_{F_{2}(z)}\right|$. $\mathcal{H} \Delta \mathcal{H}$-divergence, which is the discrepancy of two distributions $P_1, P_2$ w.r.t. any hypothesis $F_1-F_2$ where $F_1,F_2 \in \mathcal{H}$, is defined as    $d_{\mathcal{H} \Delta \mathcal{H}}(P_1,P_2)\triangleq2\sup_{F_1,F_2 \in \mathcal{H}}|\epsilon_{P_1}(F_1,F_2)-\epsilon_{P_2}(F_1,F_2)| $.
\end{definition}

\begin{theorem}[Bound of Target Domain expected risk]
The expected risk on target domain can be upper-bounded by the negative MI between features, and $\mathcal{H} \Delta \mathcal{H}$ -divergence between features of two domains:
\begin{small}
	\begin{equation}
	\begin{aligned}
	\setlength{\abovedisplayskip}{4pt}
        \setlength{\belowdisplayskip}{0cm}
	    \epsilon_{P^G_T}(F) 
	\leq \epsilon_{P^G_S}(F)-4I_{S+T}^G(Z_1;Z_2)\\+ \frac{1}{2} d_{\mathcal{H} \Delta \mathcal{H}}(P^G_S(Z),P^G_T(Z))+4H(Y).
	\end{aligned}
	\end{equation}
\end{small}

\end{theorem} 

The proof is in appendix. We give an explanation of the conditions for the upper bound to be equal. $I^G_{S+T}(Z_1;Z_2)$ is a lower bound of $I^G_{S+T}(Z;Y)$, and it measures how much uncertainty of $Y$ is reduced by knowing the feature, and it's equal to $H(Y)$ if and only if $P^G_{S+T}(Y|Z)$ is deterministic, i.e., $P^G_{S+T}(Y|Z)$ is $\delta$ distribution, which means $P^G_S(Y|Z)=P^G_T(Y|Z)=\delta_{Y(Z)}$. Thus if the $\mathcal{H} \Delta \mathcal{H}$ -divergence is zero, i.e., $P^G_S(Z)=P^G_T(Z)$, then it's ensured that $P^G_S(Z,Y)=P^G_T(Z,Y)$, and $\epsilon_{P^G_T}(F) 
	= \epsilon_{P^G_S}(F)$. 

This upper bound decomposes the cross-domain generalization error into the divergence of feature marginal distribution and MI of features. It emphasizes that in addition to the divergence of the feature marginal distributions, only a MI term is enough for knowledge transfer across domains.

In this work, we minimize the expected risk on the source domain and maximize MI, for minimizing the upper bound of expected risk on target domain. Due to the lack of labels on target domain, we estimate MI based on surrogate distribution $Q$. The expected risk upper bound based on surrogate MI is further derived as follows. 

\begin{definition}[$L_1$-distance]
    Define $L_1$-distance of ${P_1},{P_2}$ as        $d_{1}\left({P_1},{P_2}\right)\triangleq2 \sup _{B \in {\textbf{B}}}\left|\operatorname{Pr}_{P_1}[B]-\operatorname{Pr}_{P_2}[B]\right|$ where B is the set of measurable subsets under $P_1$ and $P_2$.
\end{definition}
\begin{theorem}[Bound Estimation with Surrogate Distribution]\label{maintheory}Let $B\triangleq d_1(P_T^G(Z),Q(Z))+ \epsilon_{P^G_T}(Q(Y|Z))$ be the bias of surrogate distribution $Q$ w.r.t target distribution. The expected risk on target domain can be upper-bounded by the negative surrogate MI between features, $\mathcal{H} \Delta \mathcal{H}$ -divergence between source and target domain, and additional bias of surrogate domain:
\begin{small}
	\begin{equation}
	\begin{aligned}
	\setlength{\abovedisplayskip}{4pt}
        \setlength{\belowdisplayskip}{0cm}
	    \epsilon_{P^G_T}(F) 
	\leq \epsilon_{P^G_S}(F)-4I_{S+Q}^G(Z_1;Z_2)+ B \\+\frac{1}{2} d_{\mathcal{H} \Delta \mathcal{H}}(P^G_S(Z),P^G_T(Z))+4H(Y).
	\end{aligned}
	\end{equation}
\end{small}
	
\end{theorem}

The proof is in appendix. This theorem supports the feasibility of domain adaptation via maximizing surrogate MI $I_{S+Q}^G(Z_1;Z_2)$. The bias of surrogate distribution is expressed in terms $d_1(P_T^G(Z),Q(Z))+ \epsilon_{P^G_T}(Q(Y|Z))$, where the first term is the distance between the surrogate and target feature marginal distribution, and the second term is the risk of conditional label surrogate distribution. To minimize the upper bound, the bias of the surrogate distribution should be small. 

Bias  equal to zero if and only if surrogate feature distribution and conditional label distribution are the same as target distribution, i.e., $P_T^G=Q$, where surrogate distribution does not introduce  errors.

\subsection{SIDA Framework} \label{SIDA Framework}
We employ MI maximization and surrogate distribution  in our SIDA framework, as shown in Figure  \ref{overview}. During training, a surrogate distribution is first built from target and source data via optimizing w.r.t. Laplacian and MI. Then a mixture data distribution is created by encoding source data to features and sampling target features from the surrogate distribution. The encoder is optimized by maximizing MI, and minimizing classification error. 
The overall loss is: 
\begin{small}
\begin{equation}
    \setlength{\abovedisplayskip}{4pt}
        \setlength{\belowdisplayskip}{0cm}
    L_{model} =L_{Classify}+ \alpha_1 L_{MI} + \alpha_2 L_{Auxiliary} + L_{Laplacian}.
\end{equation}
\end{small}
 We elaborate each module in the following sections, and introduce the optimization of surrogate distribution in the last sections.

\subsubsection{Mutual Information Estimation}
Several MI estimation and optimization methods are proposed in deep learning \cite{poole2019variational}. In this work, we use the following variational lower bound of MI as proposed in \cite{nguyen2010estimating}:
\begin{small}
\begin{equation}
\begin{aligned}
    \setlength{\abovedisplayskip}{4pt}
        \setlength{\belowdisplayskip}{0cm}
    I(Z_1;Z_2) \ge & \mathbb{E}_{P(z_1, z_2)}[f(z_1, z_2)]\\
    &-e^{-1} \mathbb{E}_{P(z_1)}[\mathbb{E}_{P(z_2)}\left[e^{f(z_1, z_2)}\right]],
\end{aligned}
\end{equation}
\end{small}
where $f$ is a score function in $\mathcal{Z} \times \mathcal{Z} \rightarrow R$. The equality holds when $\frac{e^{f(z_1,z_2)}}{\mathbb{E}_{P(z_1)}e^{f(z_1,z_2)}}=\frac{P(z_1|z_2)}{P(z_1)}$ and $\mathbb{E}_{P(z_1)}\mathbb{E}_{P(z_2)}e^{f(z_1,z_2)}=e$. The proof is in appendix. 
Therefore maximizing MI can be transformed into maximizing its lower bound, and the loss is:
\begin{small}
\begin{equation} \label{NWJ}
\begin{aligned}
    \setlength{\abovedisplayskip}{4pt}
        \setlength{\belowdisplayskip}{0cm}
    L_{MI}=&-\mathbb{E}_{P_{S+Q}^G(y)}\mathbb{E}_{P_{S+Q}^G(z_1|y)}\mathbb{E}_{P_{S+Q}^G(z_2|y)}[f(z_1, z_2)]\\
    & + e^{-1} \mathbb{E}_{P_{S+Q}^G(z_1)}[\mathbb{E}_{P_{S+Q}^G(z_2)}\left[e^{f(z_1, z_2)}\right]],
    \end{aligned}
\end{equation}
\end{small}
where $f(z_1,z_2)$ is constructed as $T_{m_1}^{m_2}(|z_1-z_2|_2)$. $T_{m_1}^{m_2}$ is a threshold function, i.e., $T_{m_1}^{m_2}(a)=\max(m_1,\min(m_2,a))$.

\subsubsection{Surrogate Distribution Construction}
\label{sd}
We decompose the surrogate distribution $Q(Z,Y)$ into two factors $Q(Z,Y)=Q(Y)Q(Z|Y)$, and describe the construction of two factors individually.

According to the data shift assumption, $P_T(Y)$ is similar to $P_S(Y)$, thus $Q(Y)$ should be similar to $P_S(Y)$. However, source distribution may suffer from the class imbalance problem, which will harm the performance on classes with fewer data. A common solution to this problem is class-balanced sampling, which samples data on each class uniformly. In this work, for the balance across different classes, the marginal distribution $P_S(Y)$ and $Q(Y)$ are both considered as uniform distribution. 

As for the second term, the conditional surrogate distribution $Q(Z|Y)$ is constructed by weighted sampling method. We need to construct the $Q (Z|Y)$ to calculate Eq. \ref{NWJ}, which takes the form of expectation, and only needs samples from $Q(Z|Y)$ to estimate. Instead of explicitly modeling $Q (Y |Z)$, we use the ideas of importance sampling. For each class,  the surrogate conditional distribution $Q(Z|y_j)$ is constructed by weighted sampling from target features. Thus $Q(Z|Y)$ is a distribution on target features$\{G(x_T^i)\}_{i=1}^{n_T}$, and parameterized by $W \in R^{n_T \times n_Y}$, where $n_Y$ is the number of labels: 
\begin{small}
\begin{equation}
    Q(G(x_T^i)|y_j)=W_{ij},   \textbf{ s.t. } W_{ij}\in[0,1], \sum_i W_{ij}=1,\forall j.
\end{equation}
\end{small}

Compared with pseudo-labeling, our estimation method has the following advantages:
(1) The surrogate marginal distribution of feature $Q (Z) = \sum_Y Q (Z | Y) P (Y)$ is not fixed, which enables us to select features more flexibly.
(2)The construction process of the surrogate distribution makes MI estimation $I(Z_1,Z_2)$ more convenient. Our surrogate distribution $Q (Z | Y)$ provides weights so that weighted sampling can be performed directly.  

The challenge is to optimize the sampling probability weights $W_{ij}$ so as to minimize the bias of the surrogate distribution. 
We propose to optimize this distribution via Laplacian regularization as well as MI, which is explained in details in the following section.

\subsubsection{Surrogate Distribution Loss}
Inspired by semi-supervised learning, we expect that the surrogate distribution is consistent with the clustering structure of the feature distribution, based on the assumption that the feature is well-structured and clustered according to class, regardless of domains. We employ Laplacian regularization to capture the manifold clustering structure of feature distribution.
 
 Let $A \in R^{n_T \times n_T}$ be the adjacent matrix of target features, where the entry $A_{ij}$ measures how similar $G(x_T^i)$ and $G(x_T^j)$ are, and $D=\text{Diag}({A\boldsymbol{1}})$ is the degree matrix, i.e.  $D_{ii}=\sum_j A_{ij}$ and $D_{ij}=0,\forall i \ne j$. We construct A as K-nearest graph on target features, and 
 the Laplacian regularization of $W$ is defined as 
 \begin{small}
 \begin{equation}
 \begin{aligned}
     L_{Laplacian}&=Tr(W^TLW)\\
     &=\frac{1}{2}\sum_k \sum_{i,j}A_{ij}(\frac{W_{ik}}{D_{ii}}-\frac{W_{jk}}{D_{jj}})^2,
\end{aligned}
 \end{equation}
 \end{small}
 where L is the normalized Laplacian matrix $L=I-D^{-\frac{1}{2}}AD^{-\frac{1}{2}}$. This regularization encourages $W_{ik}$ and $W_{jk}$ to be similar if feature $G(x_T^i)$ is similar to $G(x_T^j)$. It  also enables the conditional surrogate distribution to spread uniformly on a connected region.

\subsubsection{Classification and Auxiliary Loss}
The model is optimized in supervised manner on the source domain. The classification loss is the standard cross-entropy loss via class-balanced sampling.
 \begin{small}
\begin{equation}
    L_{Classify}= -\frac{1}{n_Y} \sum_y E_{P_S(x|y)} \log P(F(G(x))=y).
\end{equation}
\end{small}
And we use auxiliary classification loss on pseudo-labels from the surrogate distribution, as the classifier will benefit from label information of the surrogate distribution. We use mean square error (MSE) for pseudo-labels, which is more robust to noise than cross entropy loss.
\begin{small}
\begin{equation}
   L_{Auxiliary}=\frac{1}{n_Y} \sum_y E_{Q(x|y)} (1-P(F(G(x))=y))^2.
\end{equation}
\end{small}

\subsubsection{Optimization of Surrogate Distribution}
We optimize both $L_{\textbf{Laplacian}}$ and $L_{\textbf{MI}}$ w.r.t. $W$ for a structured and informative surrogate distribution. At the beginning of each epoch, $W$ is initialized by K-means clustering and filtered by the distance to the clustering centers, i.e. $ \widetilde{W}_{i,j} = \boldsymbol{1}_{\mu_j \text{ nearest to } G(x_i)}·\boldsymbol{1}_{d(G(x_i),\mu_j)<\theta}$, where $\mu_j$ is the j-th clustering center during clustering, and normalized as ${W}_{i,j}=\frac{\widetilde{W}_{i,j}}{\sum_i \widetilde{W}_{i,j}}$. 

To minimize two losses w.r.t $W$, the gradients are derived analytically. The derivation is in appendix.

Based on the gradient of these two losses, we perform T-step descent update of $W$ with learning rate $\eta_1$ and $\eta_2$ respectively, and each step we project $W$ back to the probability simplex. See appendix for details.

\section{Experiments}

In this section, We evaluate the proposed method on three public domain
adaptation benchmarks, compared with recent state-of-the-art
UDA methods. We conduct extensive ablation study to discuss our method. 

\subsection{Datasets} 

VisDA-2017 \cite{peng2017visda} is a challenging benchmark for UDA with the domain shift from synthetic data to real imagery. It contains 152,397 training images and 55,388 validation images across 12 classes. Following the training and
testing protocol in \cite{long2017conditional}
, the model is trained on labeled
training and unlabeled validation set and tested on the validation
set.

Office-31 \cite{saenko2010adapting} is a commonly used dataset for UDA, where images are collected from three distinct domains: Amazon (A), Webcam (W) and DSLR (D). The dataset consists of 4,110 images belonging
to 31 classes, and is imbalanced
across domains, with 2,817 images in A domain, 795 images
in W domain, and 498 images in D domain. Our method is evaluated on all six transfer tasks. We follow the standard protocol
for UDA \cite{long2017deep} 
 to use all labeled source
samples and all unlabeled target samples as the training
data.

Office-Home \cite{venkateswara2017deep} is another classical dataset with 15,500 images of 65 categories in office and home settings, consisting of 4 domains
including Artistic images (A), Clip Art images (C), Product
images (P) and Real-World images (R). Following the common protocol, all 65 categories from the four domains are used for evaluation
of UDA, forming 12 transfer
tasks.

\subsection{Implementation details}
For each transfer task, mean ($\pm$std) over 5 runs of the test accuracy are reported. We use the ImageNet pre-trained ResNet-50 \cite{he2016deep} without final classifier layer as the
encoder network $G$ for Office-31 and Office-Home, and ResNet-101 for VisDA-2017. The details of experiments are in appendix. The code is available at \href{https://github.com/zhao-ht/SIDA}{https://github.com/zhao-ht/SIDA}.

\subsection{Baselines}
We compare our approach with the state of the arts.
Domain alignment methods include  DAN \cite{long2015learning}, DANN \cite{ganin2016domain}, JAN \cite{long2017deep}. Class-level methods include conditional alignment methods (CDAN \cite{long2018conditional}, DCAN \cite{li2020domain}, ALDA \cite{chen2020adversarial}),  and contrastive methods (DRMEA \cite{luo2020unsupervised}, ETD \cite{li2020enhanced}, DADA \cite{tang2020discriminative}, SAFN \cite{xu2019larger}). We only report available results in each baseline. We use NA, DA, CA, CT to note no adaptation method, domain alignment methods, conditional alignment methods and contrastive methods respectively.

\begin{table}[htbp]
	\resizebox{\linewidth}{!}{
	\setlength{\tabcolsep}{0.5 mm}{
	\begin{tabular}{cc|ccccccccccccc}
		\hline
 	Type &	Methods & Plane & Bcycl & Bus   & Car   & Horse & Knife & Mcyle & Person & Plant & Sktbrd & Train & Truck & Avg   \\ 
 	\hline
	NA &	ResNet-101 & 55.1  & 53.3  & 61.9  & 59.1  & 80.6  & 17.9  & 79.7  & 31.2   & 81.0    & 26.5   & 73.5  & 8.5   & 52.4  \\
		\hline
		
	DA &	DAN        & 87.1  & 63.0   & 76.5  & 42.0    & 90.3  & 42.9  & 85.9  & 53.1   & 49.7  & 36.3   & 85.8  & 20.7  & 61.1  \\
	&	DANN       & 81.9  & 77.7  & 82.8  & 44.3  & 81.2  & 29.5  & 65.1  & 28.6   & 51.9  & 54.6   & 82.8  & 7.8   & 57.4  \\ 
	
	\hline
	
	CA &	CDAN       & 85.2  & 66.9  & 83.0    & 50.8  & 84.2  & 74.9  & 88.1  & 74.5   & 83.4  & 76.0     & 81.9  & 38.0    & 73.9  \\ 
	&	ALDA       & 93.8  & 74.1  & 82.4  & 69.4  & 90.6  & 87.2  & 89.0    & 67.6   & 93.4  & 76.1   & 87.7  & 22.2  & 77.8  \\
		
		\hline
	CT &	DRMEA      & 92.1  & 75.0    & 78.9  & 75.5  & 91.2  & 81.9  & 89.0    & 77.2   & 93.3  & 77.4   & 84.8  & 35.1  & 79.3  \\
		
	&	DADA & 92.9 & 74.2 & 82.5 & 65.0 & 90.9 & 93.8 & 87.2 & 74.2 & 89.9 & 71.5 & 86.5 & 48.7 & 79.8 \\
		
	&	SAFN & 93.6 & 61.3 & 84.1 & 70.6 & 94.1 & 79.0 & 91.8 & 79.6 & 89.9 & 55.6 & 89.0 & 24.4 & 76.1 \\

	    \hline

	Ours &	SIDA       & 95.4 & 83.1 & 77.1 & 64.6 & 94.5 & 97.2 & 88.7 & 78.4  & 93.8 & 89.9  & 85.2 & 59.4 & \textbf{84.0} \\ \hline
	\end{tabular}}}
	\caption{Accuracy (\%) on VisDA-2017}
	\label{visda}
\end{table}

\begin{table}[htbp]
\resizebox{\linewidth}{!}{
\setlength{\tabcolsep}{1.0 mm}{
	\begin{tabular}{cc|ccccccc}
		\hline
	Type	& Methods        & A$\rightarrow$W       & D$\rightarrow$W      & W$\rightarrow$D       & A$\rightarrow$D       & D$\rightarrow$A       & W$\rightarrow$A       & avg   \\ \hline
NA	&	ResNet-50  & 68.4$\pm$0.2 & 96.7$\pm$0.1 & 99.3$\pm$0.1 & 68.9$\pm$0.2 & 62.5$\pm$0.3 & 60.7$\pm$0.3 & 76.1\\
		
		\hline
		
	\multirow{3}{*}{DA} &	DAN   & 80.5$\pm$0.4 & 97.1$\pm$0.2  & 99.6$\pm$0.1 & 78.6$\pm$0.2 & 63.6$\pm$0.3 & 62.8$\pm$0.2 & 80.4\\
	&	DANN  & 82.0$\pm$0.4 & 96.9$\pm$0.2 & 99.1$\pm$0.1 & 79.7$\pm$0.4 & 68.2$\pm$0.4 & 67.4$\pm$0.5 & 82.2\\
	&	JAN  & 85.4$\pm$0.3 & 97.4$\pm$0.2 & 99.8$\pm$0.2 & 84.7$\pm$0.3 & 68.6$\pm$0.3 & 70.0$\pm$0.4 & 84.3\\
		\hline
		
\multirow{3}{*}{CA}	& CDAN  & 94.1$\pm$0.1 & 98.6$\pm$0.1& 100.0$\pm$0.0 & 92.9$\pm$0.2& 71.0$\pm$0.3 &69.3 $\pm$
		0.3 &87.7\\
	&	DCAN  & 95.0 & 97.5 & 100.0 & 92.6 & 77.2 & 74.9 & 89.5  \\ 
	&	ALDA      & 95.6$\pm$0.5  & 97.7$\pm$0.1 & 100.0$\pm$0.0 & 94.0$\pm$0.4  & 72.2$\pm$0.4  & 72.5$\pm$0.2  & 88.7  \\
		
		\hline
		
\multirow{3}{*}{CT}	&	ETD & 92.1 & 100.0 & 100.0 & 88.0 & 71.0 & 67.8 & 86.2 \\
		
	&	DADA & 92.3$\pm$0.1 & 99.2$\pm$0.1 & 100.0$\pm$0.0 & 93.9$\pm$0.2 & 74.4$\pm$0.1 & 74.2$\pm$0.1 & 89.0\\
		
	&	SAFN & 90.3 & 98.7 & 100.0 & 92.1 & 73.4 & 71.2 & 87.6 \\
		
		\hline

	Ours	& SIDA & 94.5$\pm$0.6 & 99.2$\pm$0.1 & 100.0$\pm$0.0   & 95.7$\pm$0.3 & 76.6$\pm$0.6 & 76.2$\pm$0.4 & \textbf{90.4} \\ \hline
	\end{tabular}}}
	       \caption{Accuracy(\%) on Office-31}
	       \label{office31}
	\vspace{-10pt}
\end{table}

\begin{table}[htbp]
    \small
	\resizebox{1\linewidth}{!}{
	\setlength{\tabcolsep}{0.5 mm}{
	\begin{tabular}{cc|ccccccccccccc}
		\hline
	Type &	Methods             & A$\rightarrow$C    & A$\rightarrow$P    & A$\rightarrow$R    & C$\rightarrow$A    & C$\rightarrow$P                     & C$\rightarrow$R    & P$\rightarrow$A    & P$\rightarrow$C    & P$\rightarrow$R    & R$\rightarrow$A                     & R$\rightarrow$C    & R$\rightarrow$P    & Avg   \\ 
		
		\hline
		
	NA &	ResNet-50 & 34.9 & 50.0 & 58.0 & 37.4 & 41.9 & 46.2 & 38.5 & 31.2 & 60.4 & 53.9 & 41.2 & 59.9 & 46.1\\
\hline
		
	\multirow{3}{*}{DA} &	DAN                & 43.6    & 57.0    & 67.9    & 45.8    & 56.5                     & 60.4    & 44.0    & 43.6    & 67.7    & 63.1                     & 51.5    & 74.3    & 56.3 \\
	&	DANN               & 45.6    & 59.3    & 70.1    & 47.0    & 58.5                     & 60.9    & 46.1    & 43.7    & 68.5    & 63.2                     & 51.8    & 76.8    & 57.6 \\
	&	JAN                & 45.9    & 61.2    & 68.9    & 50.4    & 59.7                     & 61.0    & 45.8    & 43.4    & 70.3    & 63.9                     & 52.4    & 76.8    & 58.3 
		\\ 

		\hline
	\multirow{3}{*}{CA}  &	CDAN               & 50.7    & 70.6    & 76.0    & 57.6    & 70.0                     & 70.0    & 57.4    & 50.9    & 77.3    & 70.9                     & 56.7    & 81.6    & 65.8 \\
	&	DCAN               & 54.5    & 75.7    & 81.2    & 67.4    & 74.0                     & 76.3    & 67.4    & 52.7    & 80.6    &  74.1                     & 59.1    & 83.5    & 70.5 \\
	&	ALDA               & 53.7    & 70.1    & 76.4    & 60.2    & 72.6                     & 71.5    & 56.8    & 51.9    & 77.1    & 70.2                     & 56.3    & 82.1    & 66.6 \\
		
		\hline
		 
	\multirow{3}{*}{CT} &	DRMEA & 52.3 & 73.0 & 77.3 & 64.3 & 72.0 & 71.8 & 63.6 & 52.7 & 78.5 & 72.0 & 57.7 & 81.6 & 68.1 \\

    &    ETD & 51.3 & 71.9 & 85.7 & 57.6 & 69.2 & 73.7 & 57.8 & 51.2 & 79.3 & 70.2 & 57.5 & 82.1 & 67.3\\
        
    &    SAFN & 54.4 & 73.3 & 77.9 & 65.2 & 71.5 & 73.2 & 63.6 & 52.6 & 78.2 & 72.3 & 58.0 & 82.1 & 68.5\\



		\hline
	Ours &	SIDA                & 57.2     & 79.1     & 81.7    & 67.1   &  74.5 
		  & 77.3   & 67.2    & 53.9    & 82.5    &  71.4   & 58.7    & 83.3     & \textbf{71.2} \\ \hline
	\end{tabular}}}
	\label{Tab2}

	      \caption{Accuracy (\%) on Office-Home}
	      \label{officehome}
  
\end{table}

\subsection{Results and Comparative Analysis}

In this section we will present our results and compare with other methods for evaluation on three standard benchmarks mentioned earlier. We report average classification accuracies  with standard deviations. Results of other methods are collected from original papers or the follow-up work. We provide visualizations of the features learned by the model in the appendix.

\textbf{VisDA-2017}\ \ \  Table \ref{visda} summarizes our experimental results on the  challenging VisDA-2017 dataset. For fair comparison, all methods listed here use ResNet-101 as the backbone network. Note that SIDA  outperforms baseline models with an average accuracy of 84.0, surpassing the previous best result reported by +4\%.

\textbf{Office-31}\ \ \ The unsupervised adaptation results on six Office-31 transfer tasks based on ResNet-50 are reported in Table \ref{office31}. As the data reveals, the average accuracy of SIDA is 90.4,  the best among all compared methods. It is noteworthy that our proposed method substantially improves the classification accuracy on hard transfer tasks, e.g. W$\rightarrow$A, A$\rightarrow$D, and D$\rightarrow$A, where source and target data are not similar. Our model also achieves comparable classification
performance on easy transfer tasks, e.g. D$\rightarrow$W,
W$\rightarrow$D, and A$\rightarrow$W. Our improvements are mainly on hard settings.

\textbf{Office-Home}\ \ \ Results on Office-Home using ResNet-50 backbone are reported in Table \ref{officehome}. It can be observed that SIDA exceeds all compared methods on most transfer tasks with an average accuracy of 71.2. The performance reveals the importance of maximizing MI between feature in difficult domain-adaptation tasks which contain more categories.

In summary, our surrogate MI maximization approach achieves competitive performance compared to traditional alignment based methods and recent pseudo-label based methods for UDA. It underlines the validity of using information theory methods for UDA via MI maximization.

\begin{table}[htbp]
\resizebox{\linewidth}{!}{
\setlength{\tabcolsep}{3.0 mm}{
	\begin{tabular}{cc|ccccc}
		\hline
		\small
		MI & SD & A$\rightarrow$W                       & A$\rightarrow$D                       & D$\rightarrow$A                       & W$\rightarrow$A                       & Avg   \\ \hline
		${\times}$   &  ${\times}$ & 90.25 $\pm$ 0.2                   & 92.37  $\pm$ 0.1                   & 74.21  $\pm$   0.2                & 74.09  $\pm$ 0.1                  & 82.7 \\
		${\times}$ & $\surd$     & 92.08$\pm$ 0.3 & 94.28$\pm$0.3 & 74.23$\pm$0.9 & 74.74$\pm$ 0.8 & 83.8 \\
		$\surd$    &${\times}$ & 94.03$\pm$ 0.1 & 95.28$\pm$ 0.1 & 75.86$\pm$ 0.4 & 75.72 $\pm$ 0.5 & 85.2 \\ \hline
		$\surd$    & $\surd$     & 94.52 $\pm$ 0.6                   & 95.68  $\pm$  0.1                 & 76.62 $\pm$    0.6                & 76.22 $\pm$    0.4                & 85.8 \\ \hline
	\end{tabular}}}
	\caption{Ablation Study}
	\label{ablation}
	\vspace{-10pt}
\end{table}



	

\subsection{Ablation Study}
In this section, to evaluate  how different components of our work contribute to the final performance, we conduct ablation study for SIDA on Office-31. We mainly focus on harder transfer tasks, e.g. A$\rightarrow$W , A$\rightarrow$D, D$\rightarrow$A and W$\rightarrow$A. We investigate different combinations of two components:MI maximization and surrogate distribution (SD). Note that without surrogate distribution, we use pseudo label computed by the same method as surrogate distribution initialization to estimate MI. The average classification accuracy on four tasks are in Table \ref{ablation}. 

From the results, we can observe that
the model with MI maximization outperforms the base model without the two components by about 2.5\% on average, which demonstrates the
effectiveness of the  maximization strategy. The surrogate distribution also improves the average performance by 1.1\% compared to base model, confirming that the surrogate distribution improves the estimation quality of target domain compared to pseudo label method. The combination of two components yields the highest improvement.



\section{Conclusion and Future Work}

In this work, we introduce a novel framework of unsupervised domain adaptation and provide theoretical analysis to validate our optimization objectives. Experiments show that our approach gives competitive results compared to state-of-the-art unsupervised adaptation methods on standard domain adaptation tasks. One unresolved problem is to integrate the domain discrepancy in target risk upper bound into mutual information framework. This problem is left for future work.

\bibliographystyle{named}
\bibliography{iclr2022_conference_sim}

\clearpage

\onecolumn
\section{Appendix}

\subsection{Proof for Theorem 4.1}

\begin{theorem}[Bound of Target Domain expected risk]
	The expected risk on target domain can be upper-bounded by the negative MI between features, and $\mathcal{H} \Delta \mathcal{H}$ -divergence between features of two domains:
	\begin{equation}
	\setlength{\abovedisplayskip}{4pt}
	\setlength{\belowdisplayskip}{4pt}
	\epsilon_{P^G_T}(F) 
	\leq \epsilon_{P^G_S}(F)-4I_{S+T}^G(Z_1;Z_2)+ \frac{1}{2} d_{\mathcal{H} \Delta \mathcal{H}}(P^G_S(Z),P^G_T(Z))+4H(Y)
	\end{equation}

\end{theorem} 

\begin{proof}
	
	The risk can be relaxed by triangle inequality

	\begin{equation}
	\begin{aligned}
	\epsilon_{P^G_T} (F)&\leq
	\epsilon_{P^G_T} (F')+\epsilon_{P^G_T} (F,F')\\
	&=  \epsilon_{P^G_T} (F')+\epsilon_{P^G_T} (F,F')+\epsilon_{P^G_S} (F,F')-\epsilon_{P^G_S} (F,F') \\
	&\leq \epsilon_{P^G_T} (F')+\epsilon_{P^G_T} (F,F')+\epsilon_{P^G_S} (F)+\epsilon_{P^G_S}(F')-\epsilon_{P^G_S} (F,F')\\
	&\leq \epsilon_{P^G_S} (F)+\epsilon_{P^G_T} (F')+\epsilon_{P^G_S}(F')+|\epsilon_{P^G_T} (F,F')-\epsilon_{P^G_S} (F,F')|\\
	&\leq \epsilon_{P^G_S} (F)+\epsilon_{P^G_T} (F')+\epsilon_{P^G_S}(F')+ \frac{1}{2} d_{\mathcal{H} \Delta \mathcal{H}}(P^G_S(Z),P^G_T(Z))
	\end{aligned}
	\end{equation}
	
	For the term $\epsilon_{P^G_T} (F')+\epsilon_{P^G_S}(F')$, we have 
	\begin{equation}
	\begin{aligned}
	\epsilon_{P^G_T} (F')+\epsilon_{P^G_S}(F') & = 
	\mathbb{E}_{P^G_S(z)} |P^G_S(y|z)-\delta_{F'(z)}|_1+\mathbb{E}_{P^G_T(z)} |P^G_T(y|z)-\delta_{F'(z)}|_1 \\
	& =2\mathbb{E}_{P^G_S(z)} (1-P^G_S(F'(z)|z))+2\mathbb{E}_{P^G_T(z)} (1-P^G_T(F'(z)|z)) \\
	& = 2(2-\sum_{z} P^G_S(z)(P^G_S(F'(z)|z))-\sum_{z} P^G_T(z)(P^G_T(F'(z)|z))) \\
	&= 2(2-\sum_z P^G_S(F'(z),z) -\sum_z P^G_T(F'(z),z)) \\
	&= 2(2-2\sum_z P^G_{S+T}(F'(z),z))\\
	&= 4(1-\sum_z P^G_{S+T}(z)P^G_{S+T}(F'(z)|z))\\
	&= 4(\sum_z P^G_{S+T}(z) (1-P^G_{S+T}(F'(z)|z)))\\
	&=4\mathbb{E}_{P^G_{S+T}(z)}(1-P^G_{S+T}(F'(z)|z))\\
	& \le -4\mathbb{E}_{P^G_{S+T}(z)}\log P^G_{S+T}(F'(z)|z)
	\end{aligned}
	\end{equation}
	
	While $F'$ can be any classifier, define $F'$ as the optimal classifier on $P^G_{S+T}$, i.e. $F'(z)=\arg\max_y P^G_{S+T}(y|z)$.
	
	Recall the definition of MI,
	\begin{equation}
	\begin{aligned}
	I^G_{S+T}(Z;Y)&=H_{S+T}(Y)-H_{S+T}(Y|Z)\\
	& = H_{S+T}(Y)+\mathbb{E}_{P^G_{S+T(z)}}\mathbb{E}_{P^G_{S+T}(y|z)}\log P^G_{S+T}(y|z)\\
	& \le  H_{S+T}(Y)+\mathbb{E}_{P^G_{S+T(z)}}\mathbb{E}_{P^G_{S+T}(y|z)}\log P^G_{S+T}(F'(z)|z)\\
	&= H_{S+T}(Y)+\mathbb{E}_{P^G_{S+T(z)}}\log P^G_{S+T}(F'(z)|z)
	\end{aligned}
	\end{equation}
	
	Which means that 
	
	\begin{equation}
	\begin{aligned}
	\epsilon_{P^G_T} (F')+\epsilon_{P^G_S}(F') 
	\le & -4\mathbb{E}_{P^G_{S+T}(z)}\log P^G_{S+T}(F'(z)|z) \\
	& \le -4I^G_{S+T}(Z;Y)+4H_{S+T}(Y)
	\end{aligned}
	\end{equation}
	
	According to the MI chain rule, $I (Z_1; Y, Z_2)=I (Z_1; Y)+I (Z_1; Z_2|Y)=I (Z_1; Z_2)+I (Z_1; Y | Z_2)$. Since $Z_1$ and $Z_2$ are two samples from class $Y$, $Z_1$ and $Z_2$ are independent for a given $Y$, i.e., $I (Z_1; Z_2 | Y) = 0$. So we can get $I (Z_1; Y)=I (Z_1; Z_2)+I (Z_1; Y | Z_2)$. Because $I (Z_1; Y|Z_2)\ge0$, we finally get 
	\begin{equation}
	I (Z_1; Y)\ge I (Z_1; Z_2)
	\end{equation}
	Note that $I (Z_1; Y)$ is $I (Z; Y)$, because $Z_1$ and $Z_2$ both follow distribution $P(Z|Y)$.
	
	So now we get the conclusion by
	
	\begin{equation}
	\begin{aligned}
	\epsilon_{P^G_T} (F')+\epsilon_{P^G_S}(F') 
	\le -4I^G_{S+T}(Z_1;Z_2)+4H_{S+T}(Y)
	\end{aligned}
	\end{equation}
	
	We give an explanation of the conditions for the upper bound to be equal. $I^G_{S+T}(Z_1;Z_2)$ is a lower bound of $I^G_{S+T}(Z;Y)$, and it measures how much uncertainty of $Y$ is reduced by knowing the feature, and it's equal to $H(Y)$ if and only if $P^G_{S+T}(Y|Z)$ is deterministic, i.e., $P^G_{S+T}(Y|Z)$ is $\delta$ distribution, which means $P^G_S(Y|Z)=P^G_T(Y|Z)=\delta_{Y(Z)}$. Thus if the $\mathcal{H} \Delta \mathcal{H}$ -divergence is zero, i.e., $P^G_S(Z)=P^G_T(Z)$, then it's ensured that $P^G_S(Z,Y)=P^G_T(Z,Y)$, and $\epsilon_{P^G_T}(F) 
	= \epsilon_{P^G_S}(F)$.

\end{proof}

\subsection{Proof for Theorem 4.2 }

\begin{theorem}[Bound Estimation with Surrogate Distribution]\label{maintheory}Let $B=d_1(P_T^G(Z),Q(Z))+ \epsilon_{P^G_T}(Q(Y|Z))$ be the bias of surrogate distribution $Q$ w.r.t target distribution. The expected risk on target domain can be upper-bounded by the negative surrogate MI between features, $\mathcal{H} \Delta \mathcal{H}$ -divergence between source and target domain, and additional bias of surrogate domain:
	\begin{equation}
	\setlength{\abovedisplayskip}{4pt}
	\setlength{\belowdisplayskip}{4pt}
	\epsilon_{P^G_T}(F) 
	\leq \epsilon_{P^G_S}(F)-4I_{S+Q}^G(Z_1;Z_2)+ B +\frac{1}{2} d_{\mathcal{H} \Delta \mathcal{H}}(P^G_S(Z),P^G_T(Z))+4H(Y)
	\end{equation}

\end{theorem}

\begin{proof}
	The expected expected risk can be relaxed by triangle inequality:
	
	\begin{equation}
	\begin{aligned}
	\epsilon_{P^G_T} (F)&\leq \epsilon_{P^G_T} (Q(Y|Z))+\epsilon_{P^G_T} (F,Q(Y|Z)) \\
	&\le\epsilon_{P^G_T} (Q(Y|Z))+ \epsilon_{P^G_T} (F',Q(Y|Z))+\epsilon_{P^G_T} (F,F')\\
	&=\epsilon_{P^G_T} (Q(Y|Z))+  \epsilon_{P^G_T} (F',Q(Y|Z))+\epsilon_{P^G_T} (F,F')+\epsilon_{P^G_S} (F,F')-\epsilon_{P^G_S} (F,F') \\
	&\leq\epsilon_{P^G_T} (Q(Y|Z))+ \epsilon_{P^G_T} (F',Q(Y|Z))+\epsilon_{P^G_T} (F,F')+\epsilon_{P^G_S} (F)+\epsilon_{P^G_S}(F')-\epsilon_{P^G_S} (F,F')\\
	&\leq\epsilon_{P^G_T} (Q(Y|Z))+ \epsilon_{P^G_S} (F)+\epsilon_{P^G_T} (F',Q(Y|Z))+\epsilon_{P^G_S}(F')+|\epsilon_{P^G_T} (F,F')-\epsilon_{P^G_S} (F,F')|\\
	&\leq\epsilon_{P^G_T} (Q(Y|Z))+ \epsilon_{P^G_S} (F)+\epsilon_{P^G_T} (F',Q(Y|Z))+\epsilon_{P^G_S}(F')+ \frac{1}{2} d_{\mathcal{H} \Delta \mathcal{H}}(P^G_S(Z),P^G_T(Z))\\
	&=\epsilon_{P^G_T} (Q(Y|Z))+ \epsilon_{P^G_S} (F)+\epsilon_{P^G_T} (F',Q(Y|Z))-\epsilon_{Q} (F')+ \\ 
	&\quad\epsilon_{Q} (F')+\epsilon_{P^G_S}(F')+ \frac{1}{2} d_{\mathcal{H} \Delta \mathcal{H}}(P^G_S(Z),P^G_T(Z))\\
	&\leq\epsilon_{P^G_T} (Q(Y|Z))+ \epsilon_{P^G_S} (F)+ \int(P_T^G(z)-Q(z)) |\delta_{F'(z)}-Q(Y|z)|_1 dz +\\
	&\quad\epsilon_{Q} (F')+\epsilon_{P^G_S}(F')+ \frac{1}{2} d_{\mathcal{H} \Delta \mathcal{H}}(P^G_S(Z),P^G_T(Z))\\
	&\leq\epsilon_{P^G_T} (Q(Y|Z))+ \epsilon_{P^G_S} (F)+ d_1(P_T^G(Z),Q(Z)) +\epsilon_{Q} (F')+\epsilon_{P^G_S}(F')+ \frac{1}{2} d_{\mathcal{H} \Delta \mathcal{H}}(P^G_S(Z),P^G_T(Z))\\
	&= \epsilon_{P^G_S} (F)+ B +\epsilon_{Q} (F')+\epsilon_{P^G_S}(F')+ \frac{1}{2} d_{\mathcal{H} \Delta \mathcal{H}}(P^G_S(Z),P^G_T(Z))
	\end{aligned}
	\end{equation}

	By the same method as the previous proof, terms $\epsilon_{Q} (F')+\epsilon_{P^G_S}(F')$ can be deduced into MI $-4I_{S+Q}^G(Z_1;Z_2)+4H(Y)$.
	
\end{proof}

\subsection{Proof for the Equality Condition of MI Estimation}

\begin{proposition} 
	The following MI lower bound holds
	
	\begin{equation}
	\setlength{\abovedisplayskip}{4pt}
	\setlength{\belowdisplayskip}{4pt}
	I(Z_1;Z_2) \ge \mathbb{E}_{P(z_1, z_2)}[f(z_1, z_2)]-e^{-1} \mathbb{E}_{P(z_1)}[\mathbb{E}_{P(z_2)}\left[e^{f(z_1, z_2)}\right]]
	\end{equation}
	where $f$ is arbitrary function in $\mathcal{Z} \times \mathcal{Z} \rightarrow R$. The equality holds when $\frac{e^{f(z_1,z_2)}}{\mathbb{E}_{P(z_1)}e^{f(z_1,z_2)}}=\frac{P(z_1|z_2)}{P(z_1)}$ and $\mathbb{E}_{P(z_1)}\mathbb{E}_{P(z_2)}e^{f(z_1,z_2)}=e$.
\end{proposition}
\begin{proof}
	The proof is as follows:
	
	\begin{equation}
	I(Z_1 ; Z_2)= \mathbb{E}_{P(z_1, z_2)}\left[\log \frac{q(z_1 \mid z_2)}{P(z_1)}\right] 
	+\mathbb{E}_{P(z_2)}[K L(P(z_1 \mid z_2) \| q(z_1 \mid z_2))] \ge \mathbb{E}_{P(z_1, z_2)}\left[\log \frac{q(z_1 \mid z_2)}{P(z_1)}\right]
	\end{equation}
	
	, where q is arbitrary variational distribution. Let $q(z_1 \mid z_2)=\frac{P(z_1)}{Z(z_2)} e^{f(z_1, z_2)}$, where $Z(z_2)=\mathbb{E}_{P(z_1)}\left[e^{f(z_1, z_2)}\right]$ is the normalization constant.
	
	Then 
	
	\begin{equation}
	I(Z_1 ; Z_2) \ge \mathbb{E}_{P(z_1, z_2)}[f(z_1, z_2)]-\log \mathbb{E}_{P(z_2)}[Z(z_2)]
	\end{equation}
	
	By $\log (x) \leq \frac{x}{a(x)}+\log (a(x))-1$, which is
	tight when $a(x) = x$, 
	
	\begin{equation}
	I(Z_1 ; Z_2) \ge 
	\mathbb{E}_{P(z_1, z_2)}[f(z_1, z_2)]
	-\mathbb{E}_{P(z_2)}\left[\frac{\mathbb{E}_{P(z_1)}\left[e^{f(z_1, z_2)}\right]}{a(z_2)}+\log (a(z_2))-1\right]
	\end{equation}
	
	Let $a(z_2)=e$, we get the final form of lower bound:
	
	\begin{equation}
	I(Z_1;Z_2) \ge \mathbb{E}_{P(z_1, z_2)}[f(z_1, z_2)]-e^{-1} \mathbb{E}_{P(z_2)}[\mathbb{E}_{P(z_1)}\left[e^{f(z_1, z_2)}\right]]
	\end{equation}
	
\end{proof}

\subsection{Details for Surrogate Distribution Optimization}

Let $P$  be the conditional distribution matrix of source domain, i,e $P_{ij}=P_S(x_S^i|y^j_S)$, and $W$ be the conditional distribution matrix of surrogate distribution $Q$. Let $M=\frac{1}{2}\left[\begin{array}{l}
P \\
W
\end{array}\right]$ be the conditional distribution matrix of mixture distribution $P_{S+T}$. Let $S$ is the score function matrix, i.e. $S_{i,j}=f(G(x^i),G(x^j))$, $i,j =1,\dots,n_S+n_T$, where $f$ is the score function of the MI lower bound. With class-balanced sampling,  $L_{MI}$ can be represented as follows:

\begin{equation}L_{MI}=
-(\frac{1}{n_Y} Tr(M^TSM)-\frac{1}{en_Y^2} \textbf{1}^T M^T e^S M \textbf{1})
\end{equation}

The gradient w.r.t M is 

\begin{equation}
\nabla_M L_{MI}=-2(\frac{1}{n_Y}SM-\frac{1}{en_Y^2}e^SM \textbf{1}\textbf{1}^T)
\end{equation}

And thus the gradient w.r.t $W$ is 

\begin{equation}
\nabla_W L_{MI}=-(\frac{1}{n_Y}\left(\textbf{0},I\right)SM-\frac{1}{en_Y^2}\left(\textbf{0},I\right)e^SM \textbf{1}\textbf{1}^T)
\end{equation}

Where $\left(\textbf{0},I\right) \in R^{n_T,n_S+n_T}$,$I$ is identity matrix with size  $n_T$.

In practice, we find it harmful to minimize $L_{MI}$ by $\nabla_W L_{MI}$ directly, because it will encourage the distribution to concentrate on only a few samples rapidly. We thus adjust the  decent direction of $L_{MI}$ to update the distribution slowly. Let $|\nabla_W L_{Laplacian}|$ be the entry-wise absolute value of $\nabla_W L_{Laplacian}$. The descent directions are: 
$d_1=-\nabla_W L_{Laplacian}$ and $
d_2=-|\nabla_W L_{Laplacian}|\odot(\nabla_W L_{MI})$, where $\odot$ is entry-wise multiplication. Due the property of $\nabla_W L_{Laplacian}$ which is high on the margin of each conditional distribution, this yield a diffusion-like update of $W$, which prevent rapid collapse of surrogate distribution.

Therefore, the update rule of $W$ is 
$
W^{k+1}=\Pi (W^{k}+\eta_1 d_1^k+\eta_2 d_2^k))$, where $\Pi$ is the projection operator onto probability simplex for each column of W. $\eta_1$ and $\eta_2$ are learning rate. The iteration is performed T times.

\subsection{Implementation details}

For each transfer task, mean ($\pm$std) over 5 runs of the test accuracy are reported. We use the ImageNet \cite {deng2009imagenet} pre-trained ResNet-50 \cite{he2016deep} without final classifier layer as the
encoder network $G$ for Office-31 and Office-Home, and ResNet-101 for VisDA-2017. Following \cite{kang2019contrastive}, the final classifier layer of ResNet is replaced with
the task-specific fully-connected layer to parameterize the classifier $F$, and domain-specific batch normalization parameters are used.  Code is attached in supplementary materials.

The model is trained in the finetune protocol, where the learning rate of the classifier layer is 10 times that
of the encoder, by mini-batch stochastic gradient descent (SGD) algorithm
with momentum of 0.9 and weight decay of $5e-4$. The learning rate schedule follows \cite{long2017deep,long2015learning,ganin2015unsupervised}, where the learning rate $\eta_{p}$ is adjusted following $\eta_{p}=\frac{\eta_{0}}{(1+a p)^{b}}$,
where p is the normalized training progress from 0 to 1. $\eta_{0}$ is the initial
learning rate, i.e. 0.001 for the encoder layers and 0.01
for the classifier layer. For Office-31 and Office-Home, a = 10 and
b = 0.75, while for VisDA-2017, a = 10 and b = 2.25. The coefficients of $L_{\textbf{MI}}$ and $L_{\textbf{Auxilary}}$ are $\alpha_1=0.3,\alpha_2=0.1$ for Office-31, $\alpha_1=1.3,\alpha_2=1.0$ for Office-Home, and $\alpha_1=3.0,\alpha_2=1.0$ for VisDA-2017. The hyperparameters of surrogate distribution optimization include K-nearest Graph $K=3$, number of iterations  $T=3$, learning rate $\eta_1=0.5$ and $\eta_2=0.05$.

Experiments are conducted with Python3 and Pytorch. The model is trained on single NVIDIA GeForce RTX 2080 Ti graphic card. For Office-31, each epoch takes 80 seconds and 10 seconds to perform inference. Code is attached in supplementary materials.

\subsection{Visualization}
\begin{figure}[htbp]
	
	\subfloat{\includegraphics[width=0.5\columnwidth]{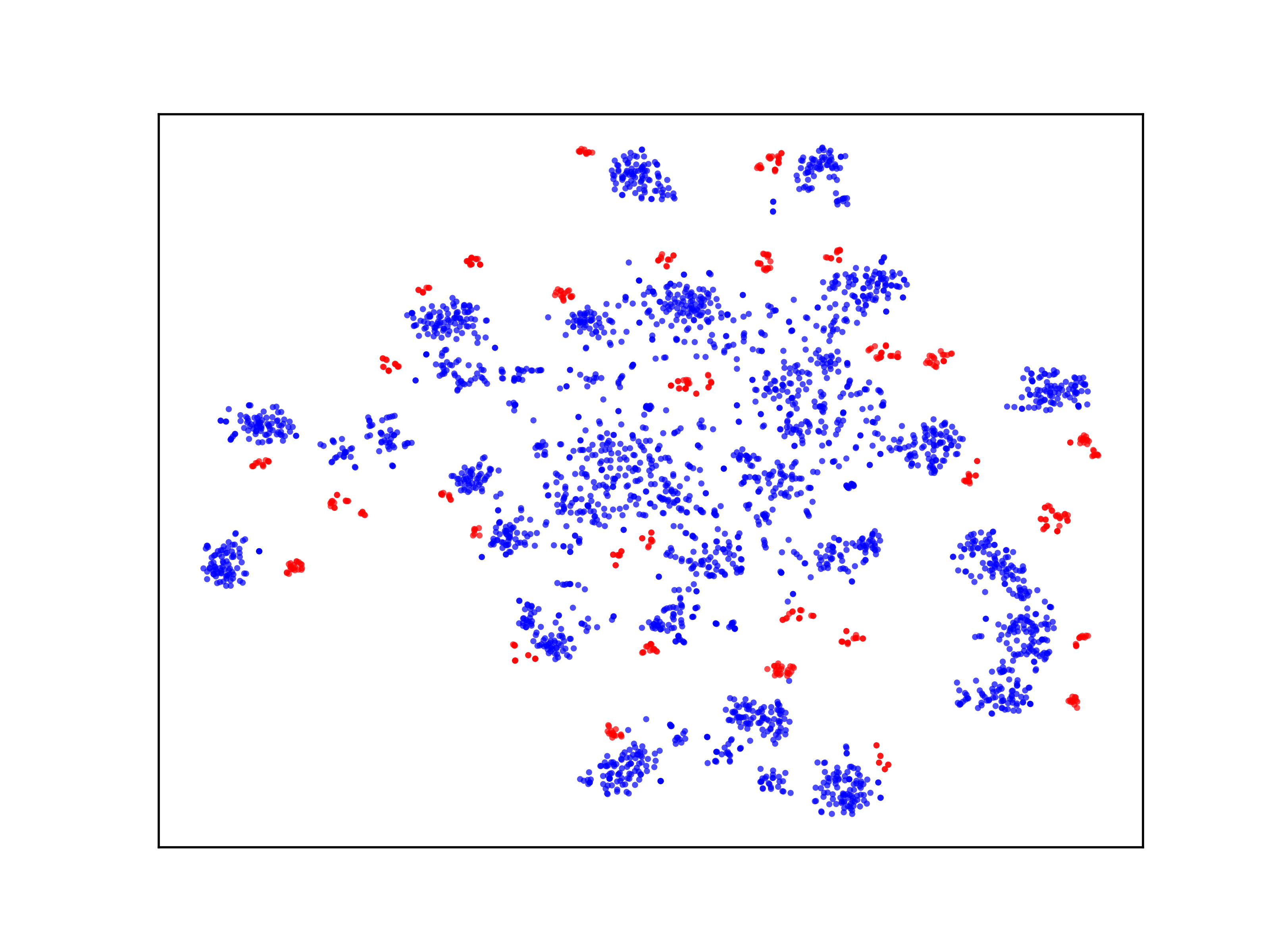}}
	\subfloat{\includegraphics[width=0.5\columnwidth]{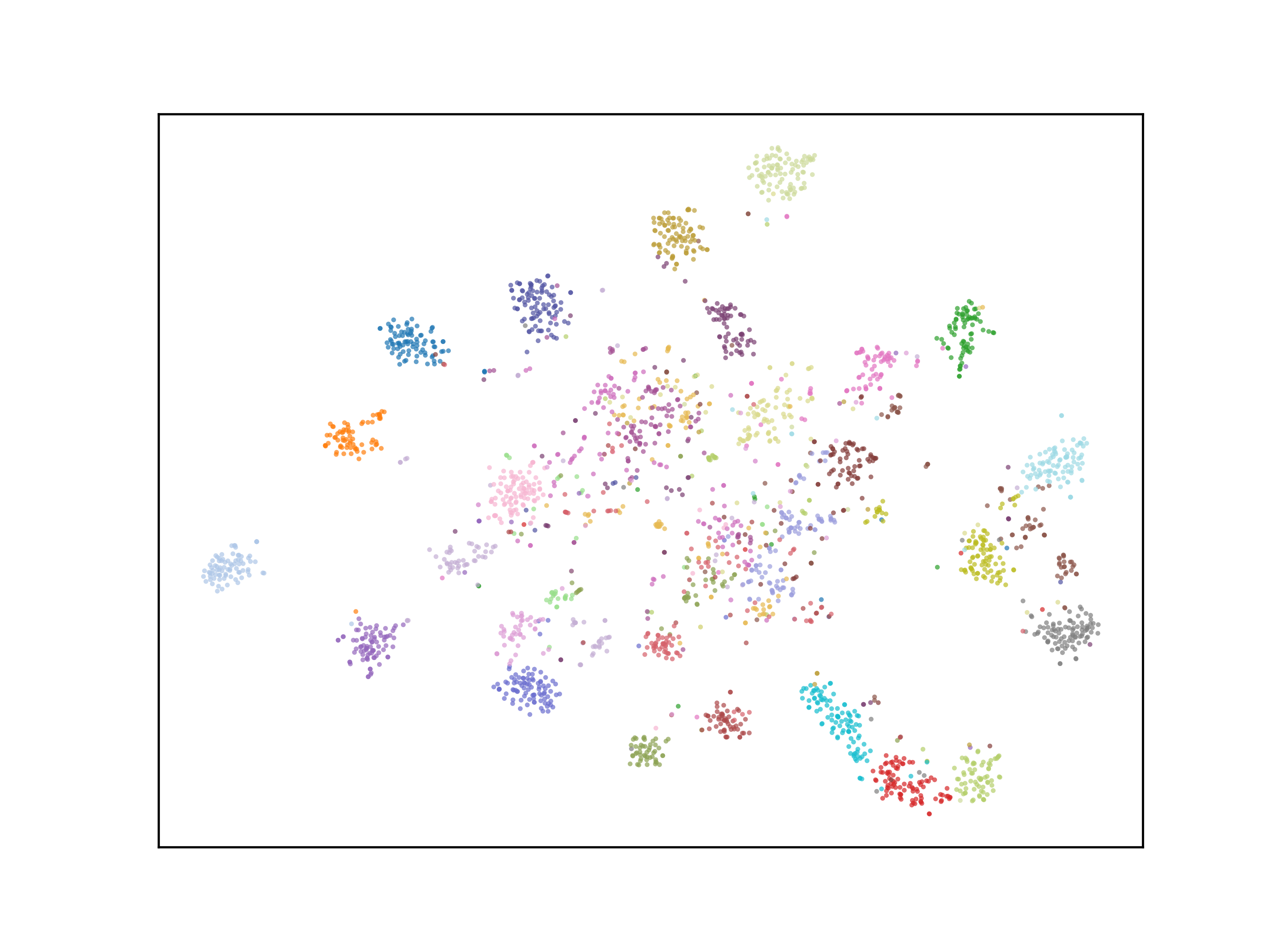}}\\
	\subfloat{\includegraphics[width=0.5\columnwidth]{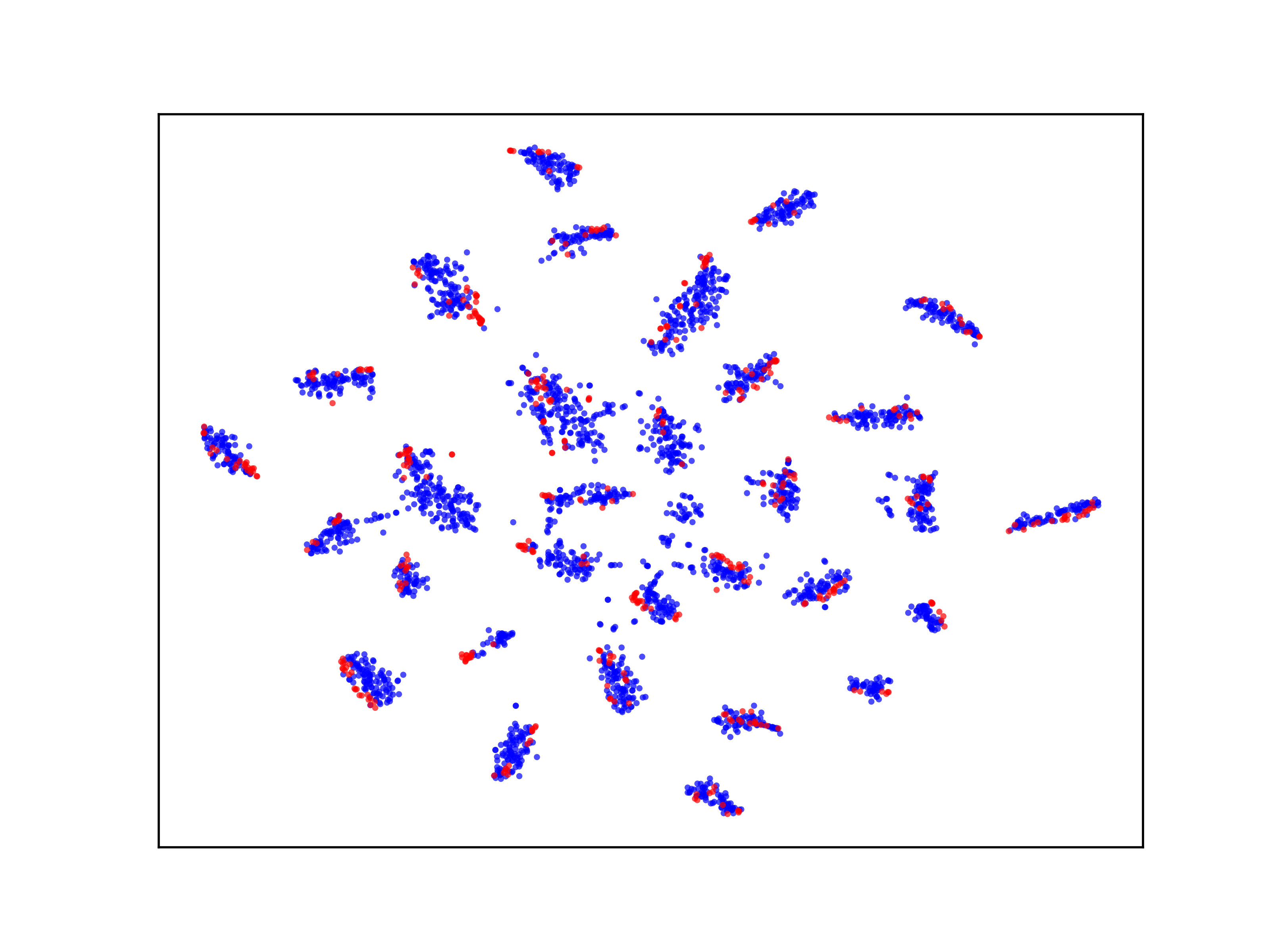}}
	\subfloat{\includegraphics[width=0.5\columnwidth]{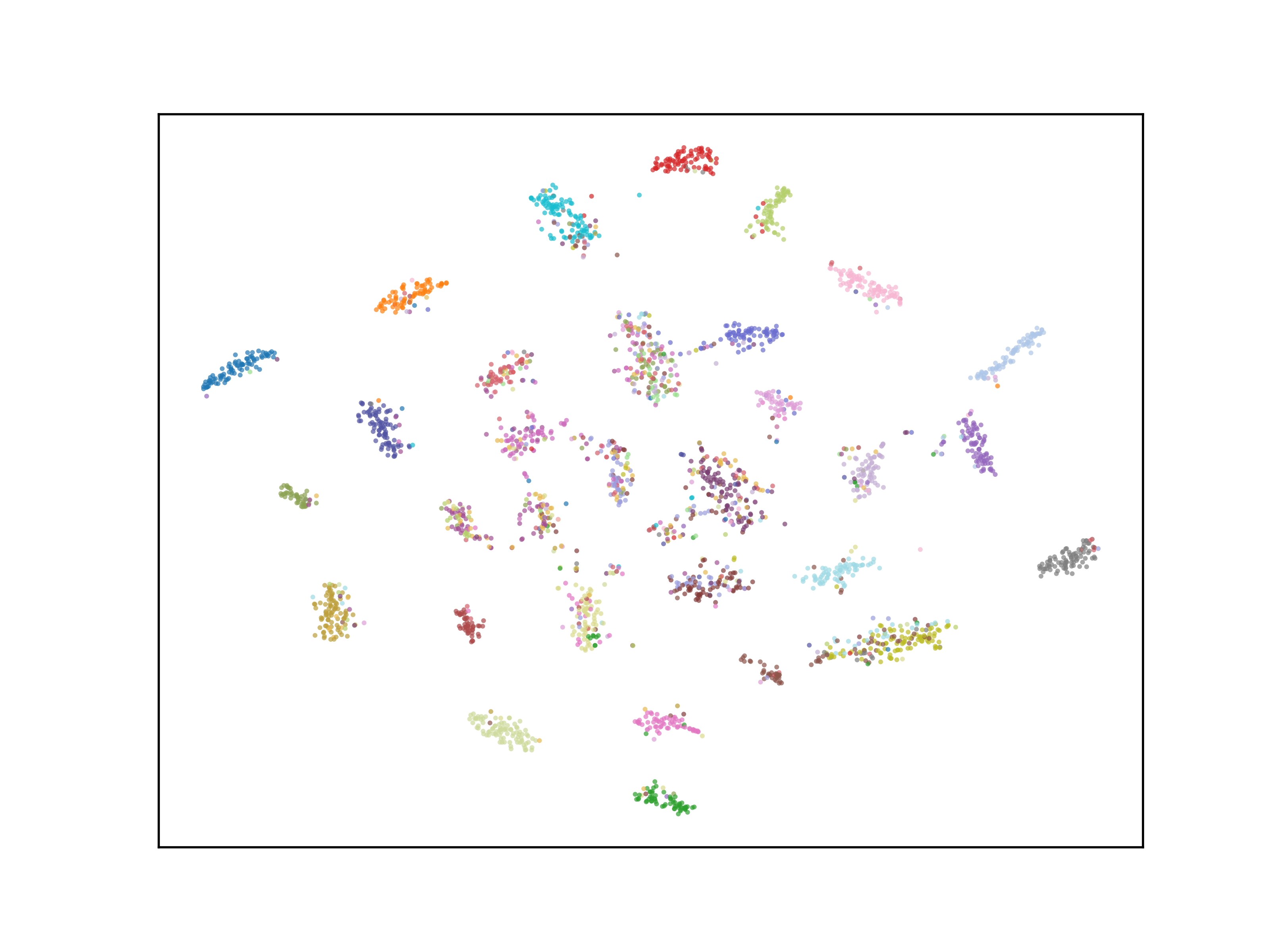}}\\
	\subfloat{\includegraphics[width=0.5\columnwidth]{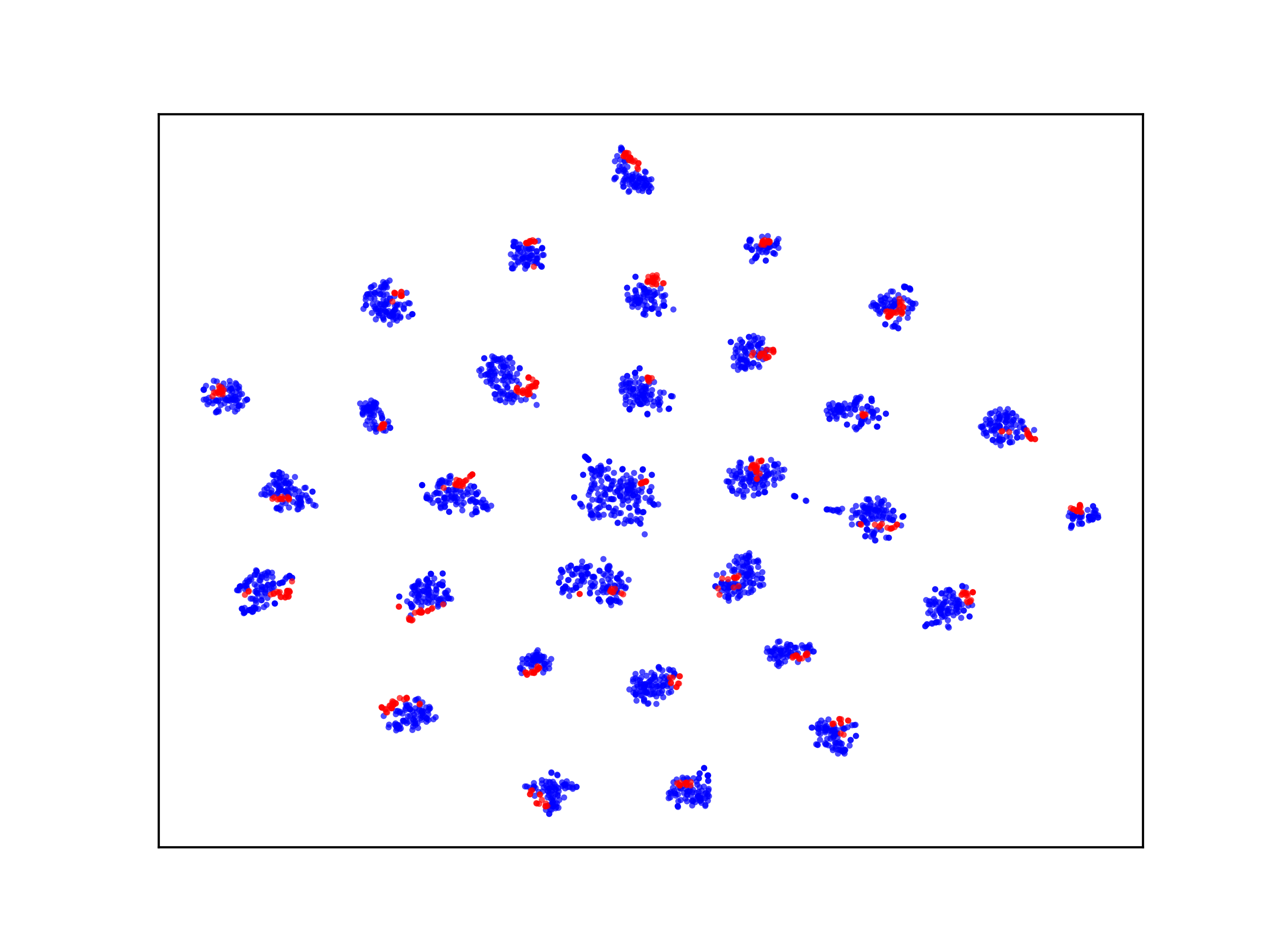}}
	\subfloat{\includegraphics[width=0.5\columnwidth]{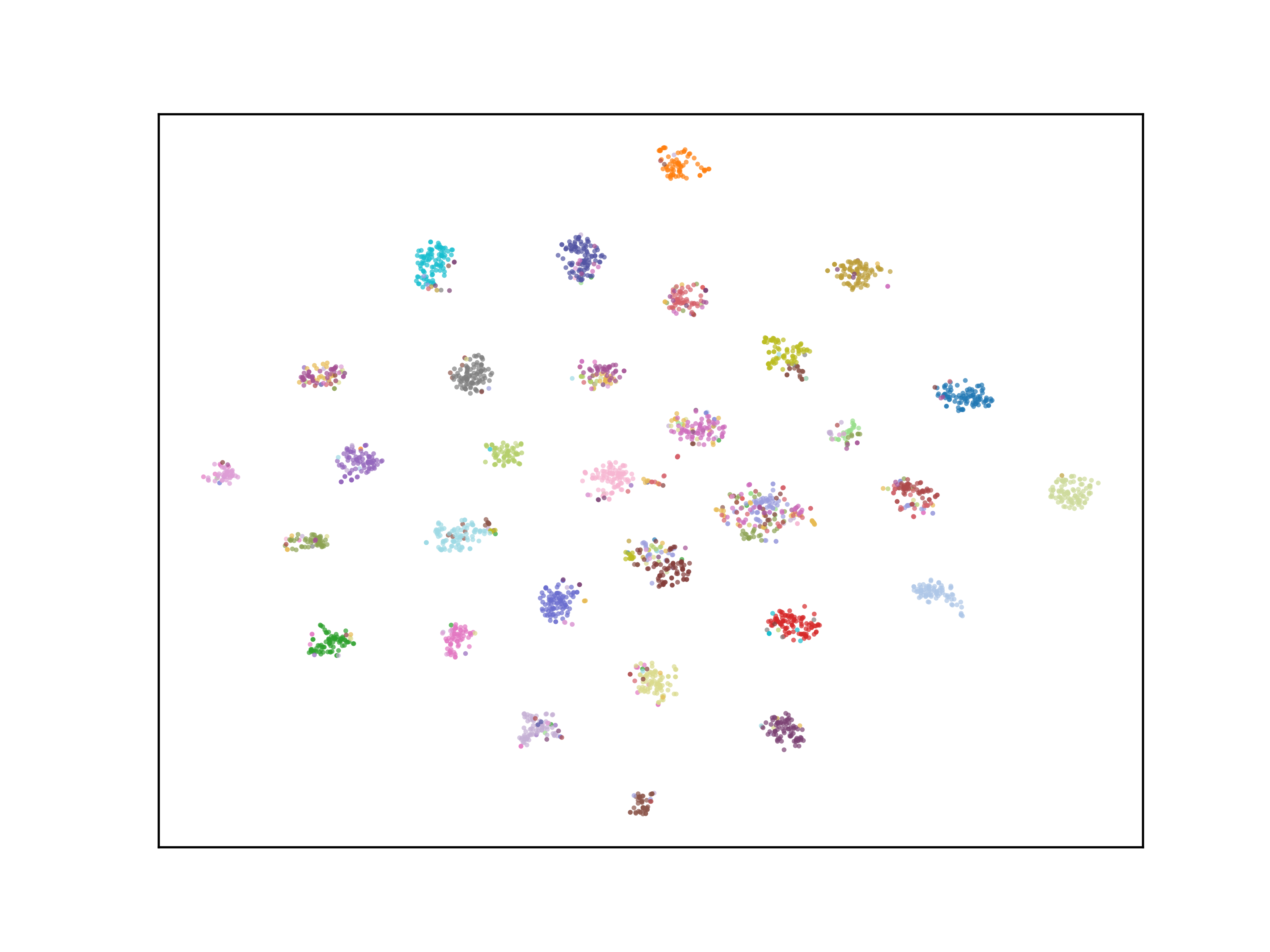}}\\
	\caption{From top to bottom are the feature visualizations on task D$\rightarrow$A of ResNet50, CDAN, and SIDA, respectively.}
	\label{D-A}
\end{figure}

\begin{figure}[htbp]
	
	\subfloat{\includegraphics[width=0.5\columnwidth]{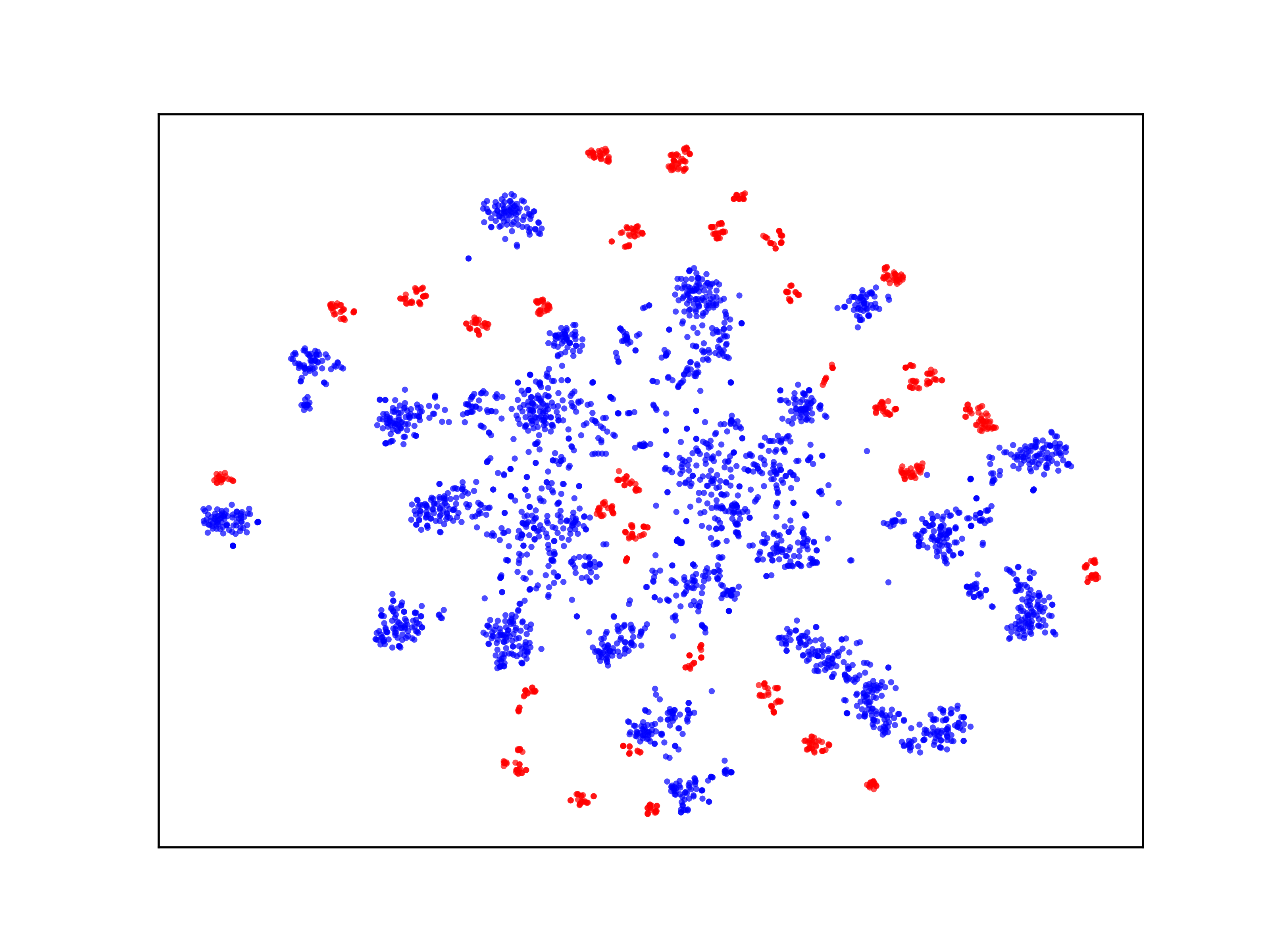}}
	\subfloat{\includegraphics[width=0.5\columnwidth]{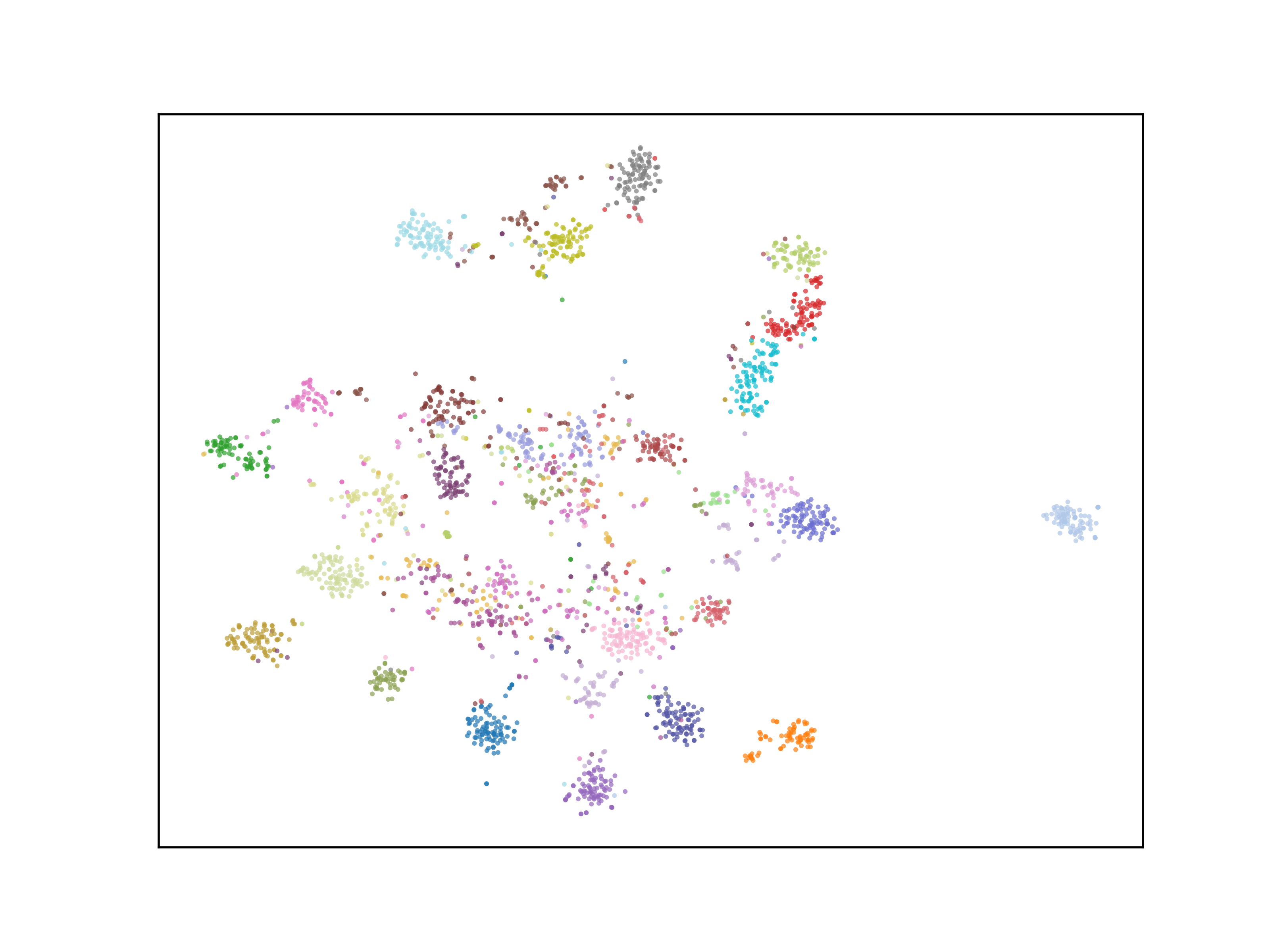}}\\
	\subfloat{\includegraphics[width=0.5\columnwidth]{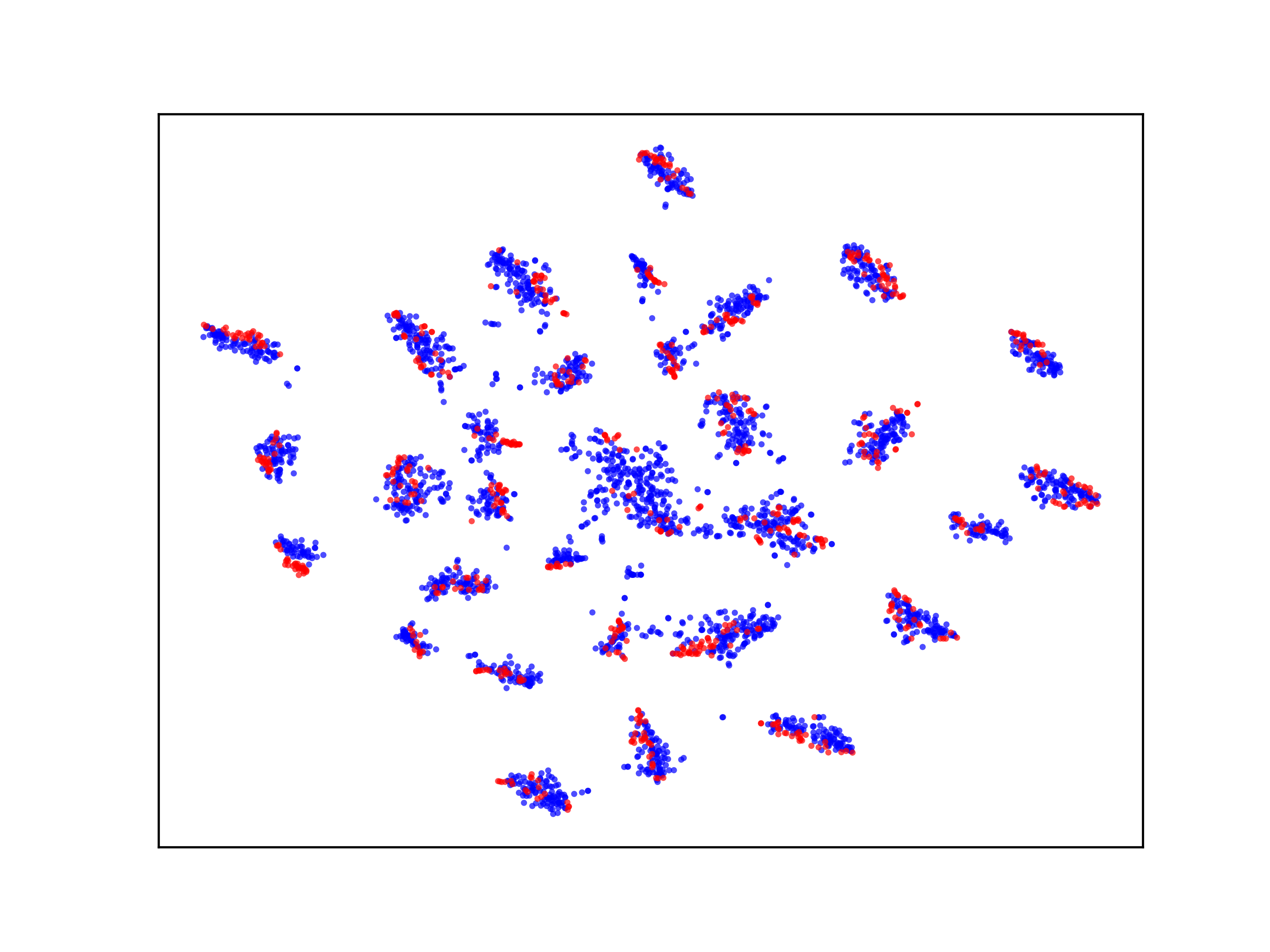}}
	\subfloat{\includegraphics[width=0.5\columnwidth]{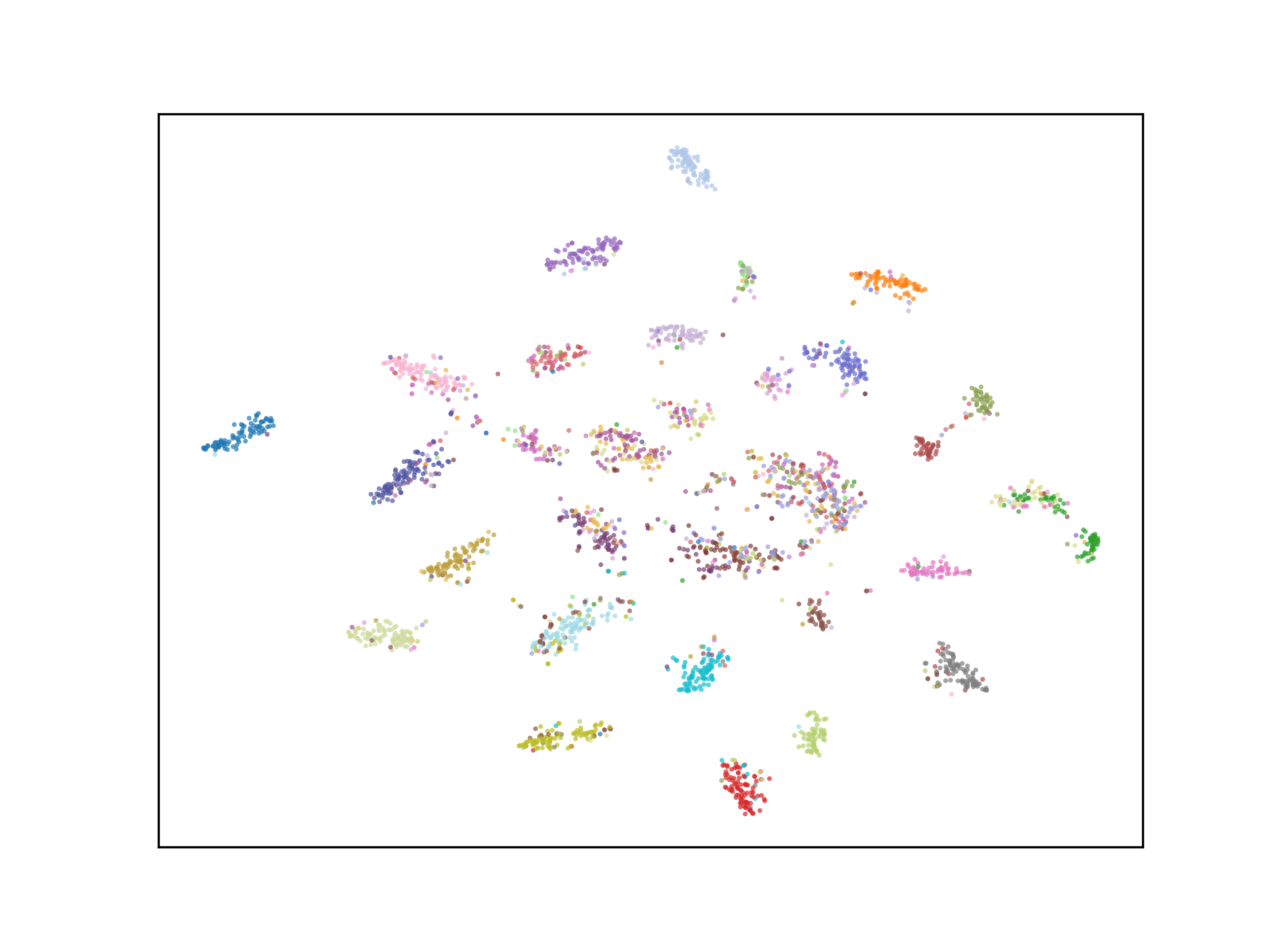}}\\
	\subfloat{\includegraphics[width=0.5\columnwidth]{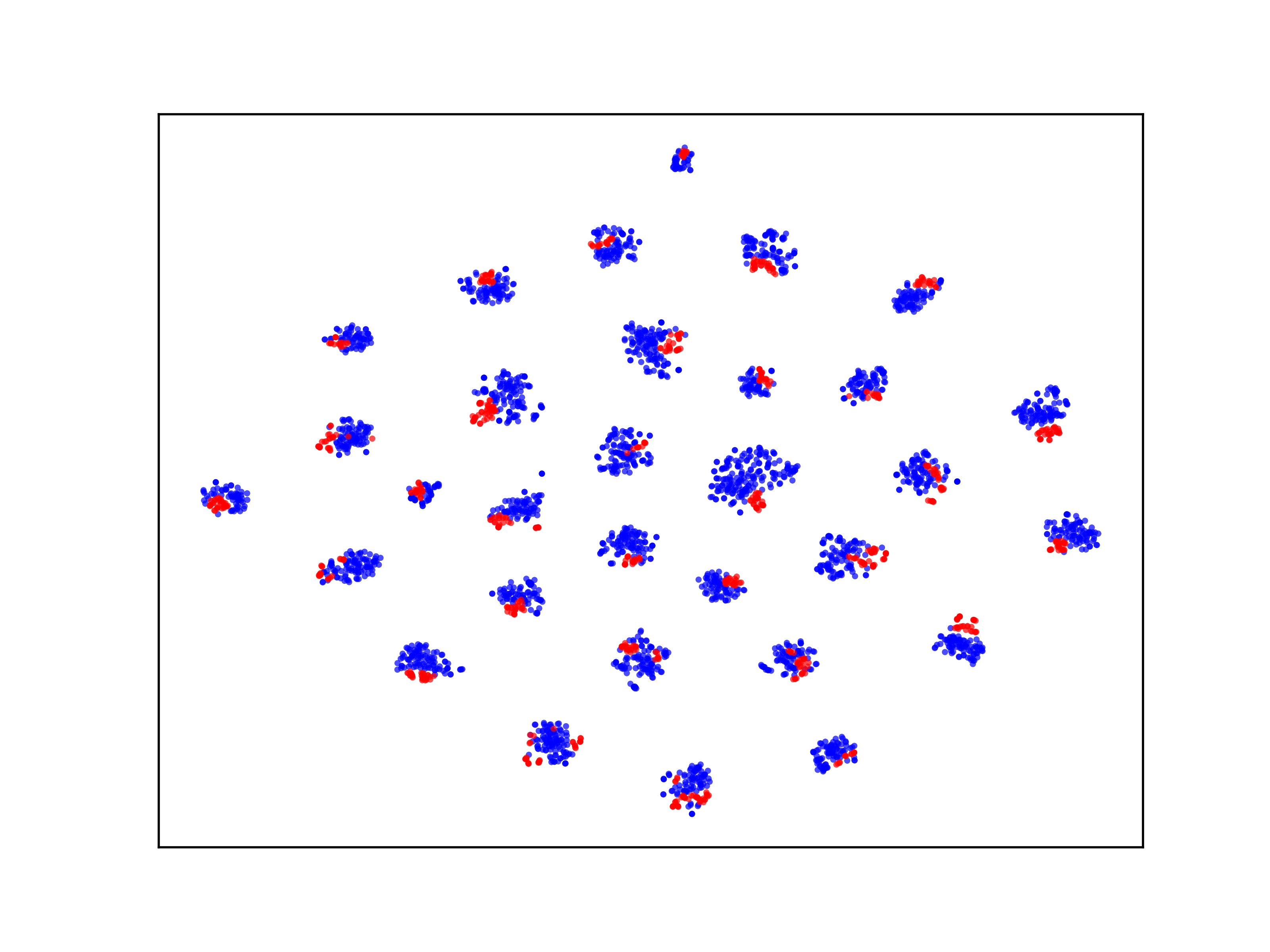}}
	\subfloat{\includegraphics[width=0.5\columnwidth]{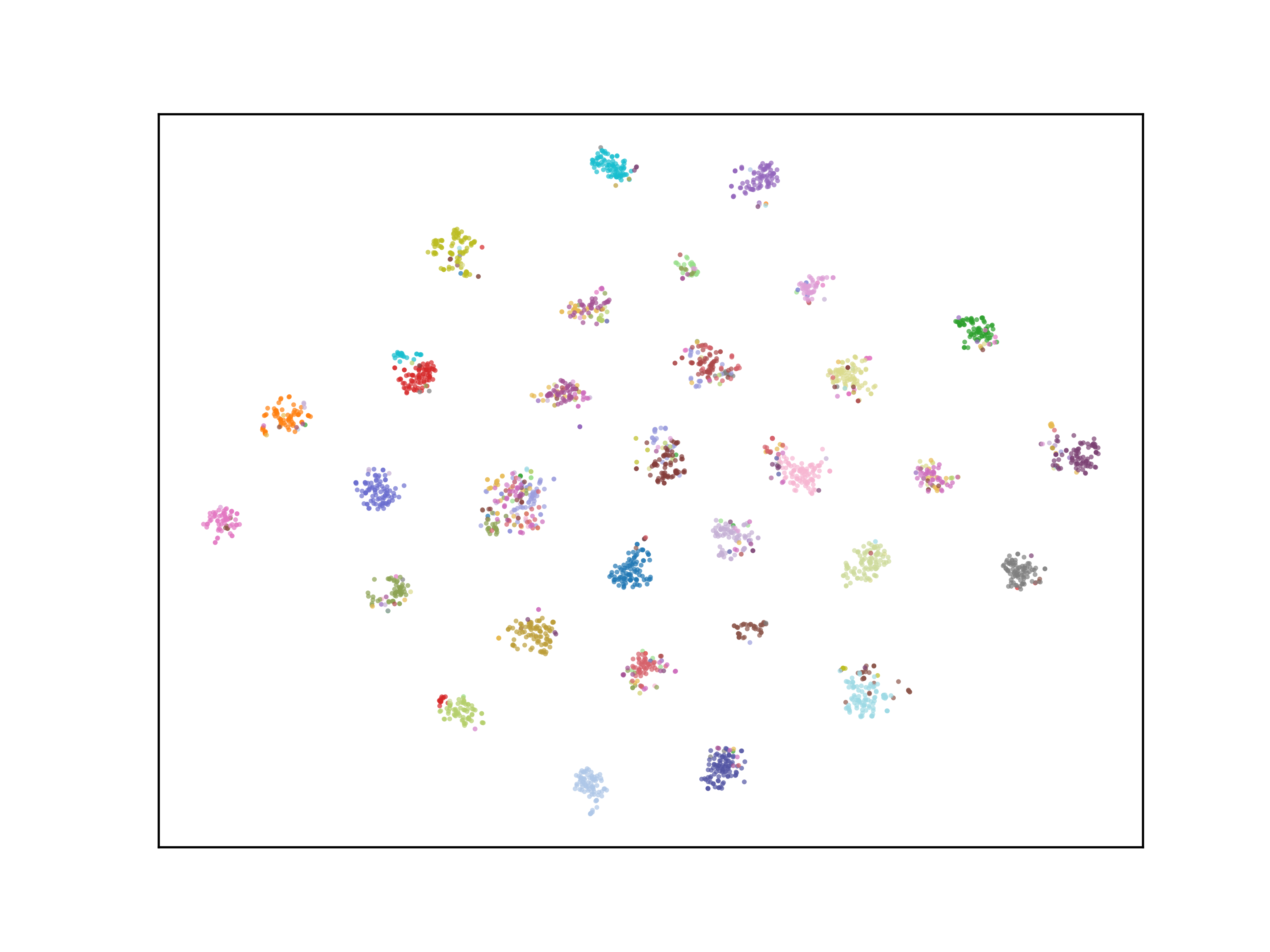}}\\
	\caption{The result of task W$\rightarrow$A is similar to task D$\rightarrow$A.}
	\label{W-A}
\end{figure}

Our visualization experiment is carried out on $D \rightarrow A$ and $W \rightarrow A$ tasks in the data set Office31, which are the two most difficult tasks in Office31. The baselines we chose were ResNet-50 pre-trained on ImageNet and CDAN\cite{long2018conditional}. We chose CDAN because it is a typical conditional domain alignment method. Pre-trained Resnet-50 is fine-tuned on the source domain and then tested on the target domain. Results of CDAN are obtained by running the official code. We train all the models until convergence, then encode the data of source domain and target domain with the model, and take representation before the final linear classification layer as feature vectors. We use t-SNE to visualize the features, using the t-SNE function of scikit-learn with default parameters. The results are in the link.

Figure \ref{D-A} shows the results of task $D \rightarrow A$. From top to bottom are the feature visualizations on task $D \rightarrow A$ of ResNet-50, CDAN, and SIDA, respectively. The left column is the feature comparison of the source and target domains. Red represents the source domain, and blue represents the target domain. The results show that SIDA emphasises discriminability of features. The right column shows the feature of different classes on the target domain. SIDA makes target features better distinguishable.

Figure \ref{W-A} shows the results of task $W \rightarrow A$. The results on task $W \rightarrow A$ are similar to task $D \rightarrow A$.

The visualization results show that SIDA can make the features of different categories more distinguishable, a natural consequence of maximizing MI among features from the same category. Thus features can be easier for classification, as the visualization shows.

\clearpage

\bibliographystyle{named}
\bibliography{iclr2022_conference}

\end{document}